\documentclass[sigconf]{acmart}

\copyrightyear{2021}
\acmYear{2021}
\setcopyright{rightsretained}
\acmConference[KDD '21]{Proceedings of the 27th ACM SIGKDD Conference on Knowledge Discovery and Data Mining}{August 14--18, 2021}{Virtual Event, Singapore}
\acmBooktitle{Proceedings of the 27th ACM SIGKDD Conference on Knowledge Discovery and Data Mining (KDD '21), August 14--18, 2021, Virtual Event, Singapore}
\acmPrice{}
\acmDOI{10.1145/3447548.3467300}
\acmISBN{978-1-4503-8332-5/21/08}
\usepackage{tikz}
\usetikzlibrary{arrows.meta}
\DeclareMathOperator{\diag}{diag}
\DeclareMathOperator{\cov}{cov}
\DeclareMathOperator{\E}{E}

\DeclareMathOperator{\PMI}{PMI}
\DeclareMathOperator{\vol}{vol}
\usepackage{mathtools}
\DeclarePairedDelimiter\norm{\lVert}{\rVert}%
\DeclareMathOperator{\rank}{rank}
\DeclareMathOperator*{\argmin}{arg\,min}

\newtheorem{observation}{Observation}
\usepackage[normalem]{ulem}
\usepackage{multirow}
\usepackage{float}
\usepackage{subfig}
\usepackage{enumitem}
\usepackage{multicol}

\begin{document}
\title{A Broader Picture of Random-walk Based Graph Embedding}
\author{Zexi Huang}
\affiliation{%
 \institution{University of California}
 \city{Santa Barbara}
 \state{CA}
 \country{USA}
}
\email{zexi_huang@cs.ucsb.edu}

\author{Arlei Silva}
\affiliation{%
 \institution{Rice University}
 \city{Houston}
 \state{TX}
 \country{USA}}
\email{arleilps@gmail.com}

\author{Ambuj Singh}
\affiliation{%
 \institution{University of California}
 \city{Santa Barbara}
 \state{CA}
 \country{USA}}
\email{ambuj@cs.ucsb.edu}

\renewcommand{\shortauthors}{Huang, et al.}


\begin{abstract}
Graph embedding based on random-walks supports effective solutions for many graph-related downstream tasks. However, the abundance of embedding literature has made it increasingly difficult to compare existing methods and to identify opportunities to advance the state-of-the-art. Meanwhile, existing work has left several fundamental questions---such as how embeddings capture different structural scales and how they should be applied for effective link prediction---unanswered. This paper addresses these challenges with an analytical framework for random-walk based graph embedding that consists of three components: a random-walk process, a similarity function, and an embedding algorithm. Our framework not only categorizes many existing approaches but naturally motivates new ones. With it, we illustrate novel ways to incorporate embeddings at multiple scales to improve downstream task performance. We also show that embeddings based on autocovariance similarity, when paired with dot product ranking for link prediction, outperform state-of-the-art methods based on Pointwise Mutual Information similarity by up to 100\%.
\end{abstract}



\maketitle
\section{Introduction}

Random-walk based graph embedding enables the application of classical algorithms for high-dimensional data to graph-based downstream tasks (e.g., link prediction, node classification, and community detection). These embedding methods learn vector representations for nodes based on some notion of topological similarity (or proximity). Since DeepWalk \cite{perozzi2014deepwalk}, we have witnessed a great interest in graph embedding both by researchers and practitioners. Several regular papers \cite{grover2016node2vec,ou2016asymmetric,qiu2018network,schaub2019multiscale,tang2015line,qiu2019netsmf,zhou2017scalable,khosla2019node,tsitsulin2018verse,cao2015grarep,chen2017harp,donnat2018learning,xin2019marc,chanpuriya2020infinitewalk} and a few surveys \cite{hamilton2017representation,cui2018survey,cai2018comprehensive,goyal2018graph,chami2020machine} have attempted to not only advance but also consolidate our understanding of these models. However, the abundance of literature also makes it increasingly difficult for practitioners and newcomers to the field to compare existing methods and to contribute with novel ones.

On the other hand, despite the rich literature, several fundamental questions about random-walk based graph embedding still remain unanswered. One such question is \textit{(Q1) how do embeddings capture different structural scales?} Random-walks of different lengths are naturally associated with varying scales \cite{delvenne2010stability}. However, downstream task performance of embeddings has been shown to be insensitive to random-walk lengths \cite{perozzi2014deepwalk, grover2016node2vec}. 
Another relevant question is \textit{(Q2) how should random-walk embeddings be used for link prediction?} Following node2vec \cite{grover2016node2vec}, several works train classifiers to predict missing links based on a set of labelled pairs (edges and non-edges) \cite{wang2017signed, nguyen2018continuous, javari2020rose}. This is counter-intuitive given that the embedding problem is often defined in terms of dot products. 
In fact, dot products are sometimes also applied for link prediction and for the related task of network reconstruction, where the entire graph is predicted based on the embeddings \cite{wang2016structural,ou2016asymmetric,goyal2018graph,zhang2018arbitrary}.

With these questions in mind, we shall take a closer look at how embeddings are produced in random-walk based methods. It starts with selecting a \emph{random-walk process}. DeepWalk \cite{perozzi2014deepwalk} applies standard random-walks, while node2vec \cite{grover2016node2vec} considers biased random-walks and APP \cite{zhou2017scalable} adopts rooted PageRank, among others. Then, a \emph{similarity function} maps realizations of the random-walk process into real values that represent some notion of node proximity. Most existing methods rely on the skip-gram language model \cite{mikolov2013distributed}, which has been shown to capture the Pointwise Mutual Information (PMI) similarity \cite{levy2014neural}. Finally, an \emph{embedding algorithm} is used to generate vector representations that preserve the similarity function via optimization. This can be based on either sampled random-walks and gradient descent (or one of its variations) \cite{perozzi2014deepwalk,grover2016node2vec}, or (explicit) matrix factorization (e.g., Singular Value Decomposition) \cite{qiu2018network,qiu2019netsmf}.

This breakdown of random-walk based embedding methods allows us to build a simple, yet powerful analytical framework with three major components: \textit{random-walk process}, \emph{similarity function}, and \textit{embedding algorithm}. As shown in \autoref{fig:embedding_methods}, our framework both categorizes existing approaches and facilitates the exploration of novel ones. 
For example, we will consider embeddings based on the autocovariance similarity, which has been proposed for multiscale community detection \cite{delvenne2010stability,delvenne2013stability}, and compare it against the more popular PMI similarity on multiple downstream tasks. 

\begin{figure*}
    \centering
    \tikzstyle{component}  = [rectangle, text centered,  text width=5em]
    \tikzstyle{arrow} =  [draw, ->, >=stealth]
    \tikzstyle{line} =[draw, ]
        \begin{tikzpicture}[
                    level 1/.style={sibling distance=26em, level distance=5em},
                    level 2/.style={sibling distance=13em, level distance=5em}, 
                    level 3/.style={sibling distance=5.5em, level distance=6em},
                    parent anchor=south,
                    child anchor=north,
                    align=center,
                    ]
        \node{Random-walk based embedding methods}
            child { 
                child {
                    child {
                        node {DeepWalk\cite{perozzi2014deepwalk}\\ 
                        Walklets\cite{perozzi2017don}\\
                        LINE\cite{tang2015line}}
                        edge from parent node(E)[left,draw=none]{Sampling}}
                    child {
                        node {NetMF\cite{qiu2018network}\\ NetSMF\cite{qiu2019netsmf}\\
                        InfiniteWalk\cite{chanpuriya2020infinitewalk}}
                        edge from parent node[right,draw=none]{Factorization}}
                    edge from parent node(S)[left=5,draw=none]{PMI}}
                child {
                    child {
                        node {\textbf{This paper}}
                        edge from parent node[left,draw=none]{Sampling}}
                    child {
                        node {Multiscale\cite{schaub2019multiscale}\\\textbf{This paper}}
                        edge from parent node[right,draw=none]{Factorization}}
                    edge from parent node[right=5,draw=none]{Autocovariance}}
                edge from parent node(RW)[left=10,draw=none]{Standard}}
            child {
                child {
                    child {
                        node {node2vec\cite{grover2016node2vec}\\APP\cite{zhou2017scalable}\\NERD\cite{khosla2019node}\\ \textbf{This paper}}
                        edge from parent node[left,draw=none]{Sampling}}
                    child {
                        node {NetMF\cite{qiu2018network}\\ \textbf{This paper}}
                        edge from parent node[right,draw=none]{Factorization}}
                    edge from parent node[left=5,draw=none]{PMI}}
                child {    
                    child {
                        node {\textbf{This paper}}
                        edge from parent node[left,draw=none]{Sampling}}
                    child {
                        node {\textbf{This paper}}
                        edge from parent node[right,draw=none]{Factorization}}
                    edge from parent node[right=5,draw=none]{Autocovariance}}
                edge from parent node[right=10,draw=none]{Non-standard}};

        \node (ET) [component, left of=E, node distance=5.2em,rotate=90] {Algorithm};
        \node (SM) [component, above of=ET, node distance=5.1em,rotate=90] {Similarity};
        \node (RWP) [rectangle, text centered, text width=6em, above of=SM, node distance=5em,rotate=90] {Process};
        \end{tikzpicture}
    \caption{Different random-walk based embedding methods (old and new) classified according to our analytical framework---with process, similarity, and algorithm as main components. A key contribution of this paper is to integrate autocovariance as a similarity metric and show that it outperforms Pointwise Mutual Information (PMI) in link prediction.} 
    \label{fig:embedding_methods}
\end{figure*}
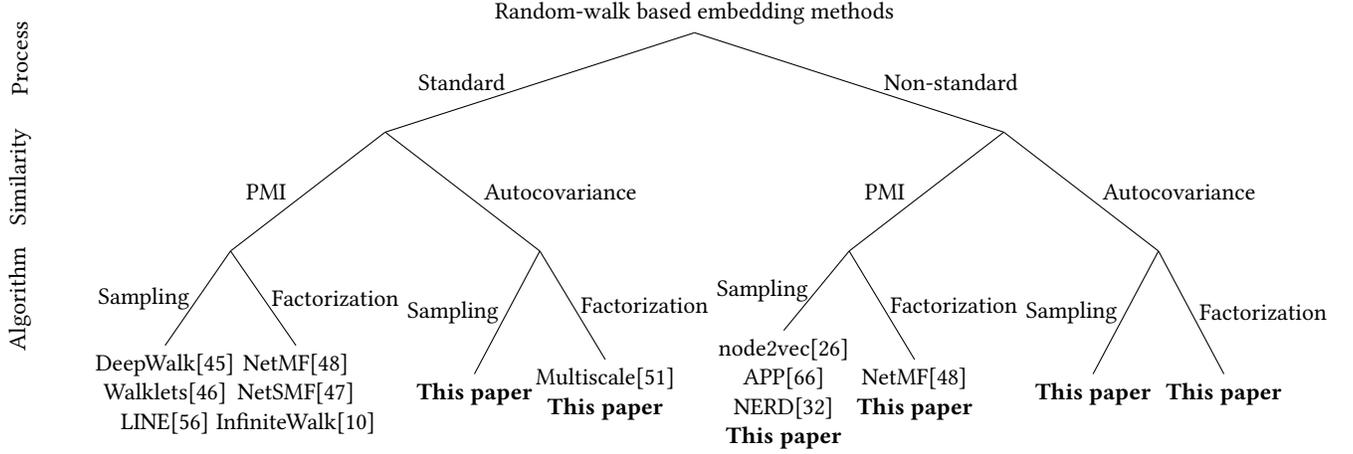

Our framework also provides tools to answer the aforementioned questions. For Q1, we will not only illustrate how past work implicitly combines multiple scales and \emph{its implications to node-level tasks}, but also propose novel ways to incorporate scales to \emph{improve edge-level task performance}. 
To answer Q2, we will show that to optimize performance, \emph{embedding methods should be designed with task settings in mind}. Specifically, we find that embeddings based on autocovariance, when paired with dot products, lead to a two-fold improvement over existing methods. Our analysis shows the reason to be that the particular combination enables the embedding to capture heterogeneous degree distributions \cite{barabasi2003scale} in real graphs.
One could argue that link prediction is the most relevant downstream task for (positional) node embeddings, as graph neural networks often outperform embedding approaches in node classification \cite{srinivasan2019equivalence}. 

To summarize, the main contributions of our paper are:
\begin{itemize}
    \item We present a unified view of random-walk based graph embedding, incorporating different random-walk processes, similarity functions, and embedding algorithms. 
    
    \item We show how autocovariance can be applied as a similarity function to create novel embeddings that
    outperform state-of-the-art methods using PMI by up to 100\%.
    
    \item We illustrate ways to exploit the multiscale nature of random-walk similarity to further optimize embedding performance.
    
    \item We conduct an extensive experimental evaluation of our framework on various downstream tasks and provide theoretical insights behind the results. 
\end{itemize}

\section{Method}
\label{sec::emb_undirected}

Consider an undirected weighted graph $\mathcal{G}=(\mathcal{V}, \mathcal{E})$, where $\mathcal{V}=\{1, \ldots n\}$ denotes the set of $n$ nodes and $\mathcal{E}$ denotes the set of $m$ edges. The graph is represented by a weighted symmetric adjacency matrix $A\in \mathbb{R}^{n\times n}$, with $A_{uv}>0$ if an edge of weight $A_{uv}$ connects nodes $u$ and $v$, and $A_{uv}=0$, otherwise. The (weighted) degree of node $u$ is defined as $\deg(u) = \sum_{v} A_{uv}$. 

A node embedding is a function $\phi: \mathcal{V} \mapsto \mathbb{R}^d$ that maps each node $v$ to a $d$-dimensional ($d \ll n$) vector $\mathbf{u}_v$. 
We refer to the embedding matrix of $\mathcal{V}$ as $U = (\mathbf{u}_1, \dotsc, \mathbf{u}_n)^T \in \mathbb{R}^{n\times d}$. For some embedding algorithms, another embedding matrix $V \in \mathbb{R}^{n\times d}$ is also generated. 
Random-walk based embedding methods use a random-walk process to embed nodes $u$ and $v$ such that a similarity metric is preserved by dot products $\mathbf{u}_u^T\mathbf{u}_v^{}$ (or $\mathbf{u}_u^T\mathbf{v}_v^{}$). In the next section, we will formalize random-walks on graphs.

\subsection{Random-walk Process}
\label{subsec::random-walk}
A random-walk is a Markov chain over the set of nodes $\mathcal{V}$. The transition probability of the walker jumping to node $v$ is based solely on its previous location $u$ and is characterized by the adjacency matrix $A$. For the \textit{standard random-walk process}, the transition probability is proportional to the edge weight $A_{uv}$:
\begin{equation}
    p(x(t{+}1){=}v|x(t){=}u)= \frac{A_{uv}}{\deg(u)}
\end{equation}
where $x(t) \in \mathcal{V}$ is the location of the walker at time $t$.  

Transition probabilities between all pairs of nodes are represented by a transition matrix $M\in \mathbb{R}^{n\times n}$: 
\begin{equation}
    M = D^{-1}A
\end{equation}
where $D = \diag([\deg(1), \dotsc, \deg(n)]) \in \mathbb{R}^n$ is the degree matrix.

For a connected non-bipartite graph, the standard random-walk is ergodic and has a unique stationary distribution $\pi \in \mathbb{R}^n$:
\begin{equation}
    \pi_u = \frac{\deg(u)}{\sum_{v} \deg(v)}
\end{equation}

Standard random-walks provide a natural way to capture node neighborhood in undirected connected graphs. One can also design \textit{biased random-walks} to explore different notions of neighborhood \cite{grover2016node2vec}. For directed graphs, a \textit{PageRank process} \cite{page1999pagerank} is often applied in lieu of standard random-walks to guarantee ergodicity. 

\subsection{Similarity Function}
\label{sec::undirected_similarity}
A node similarity metric is a function $\varphi: \mathcal{V}\times \mathcal{V} \mapsto \mathbb{R}$ that maps pairs of nodes to some notion of topological similarity. Large positive values mean that two nodes are similar to each other, while large negative values indicate they are dissimilar. 

Random-walk based similarity functions are based on co-visiting probabilities of a walker. 
An important property of random-walks that has been mostly neglected in the embedding literature is their ability to capture similarity at different structural scales (e.g., local vs. global). This is achieved via a Markov time parameter $\tau \in \mathbb{Z}_+$, which corresponds to the distance between a pair of nodes in a walk in terms of jumps. One of the contributions of this paper is to show the effect of different Markov time scales on the embedding and ways to exploit them to optimize downstream task performance.

We will describe two random-walk based similarity functions in this section, \emph{Pointwise Mutual Information (PMI)} and \emph{autocovariance}. PMI has become quite popular in the graph embedding literature due to its (implicit) use by word2vec \cite{mikolov2013distributed,levy2014neural} and DeepWalk \cite{perozzi2014deepwalk,qiu2018network}. On the other hand, autocovariance is more popular in the context of multiscale community detection \cite{delvenne2010stability,delvenne2013stability}. As one of the contributions of this paper, we will demonstrate the effectiveness of autocovariance based embeddings on edge-level tasks. 
\subsubsection{Pointwise Mutual Information (PMI)} Denote $X_v(t) \in \{0,1\}$ as the event indicator of $x(t)=v$, which can be true (1) or false (0). The PMI between events $X_u(t) =1$ and $X_v(t+\tau) = 1$ is defined as:
\begin{equation}
    \begin{aligned}
    R_{uv}(\tau) &= \PMI (X_u(t) =1, X_v(t+\tau) = 1)\\
    &= \log \frac{p(X_u(t) =1, X_v(t+\tau) = 1)}{p(X_u(t)=1)p(X_v(t+\tau)=1)}
    \end{aligned}
\end{equation}
PMI provides a non-linear information theoretic view of random-walk based proximity. For an ergodic walk with the stationary distribution $\pi$ as starting probabilities, PMI can be computed based on $\pi$ and the $\tau$-step transition probability:
\begin{equation}
    R_{uv}(\tau) = \log (\pi_u p(x(t{+}\tau){=}v|x(t){=}u)) - \log (\pi_u \pi_v)
\end{equation}
where $R_{uv}(\tau)\in [-\infty,-\log(\pi_v)]$. In matrix form:
\begin{equation}
    R(\tau) = \log (\Pi M^{\tau}) - \log(\pi \pi^T)
    \label{eqn::pmi_matrix}
\end{equation}
where $\Pi = \diag(\pi) \in \mathbb{R}^{n \times n}$ and $\log(\cdot)$ is the element-wise logarithm. The PMI matrix is symmetric due to the time-reversibility of undirected random-walks.

It is noteworthy that LINE \cite{tang2015line} and DeepWalk \cite{perozzi2014deepwalk}---two random-walk embedding methods---implicitly factorize PMI, as follows:
\begin{align}
    R_{LINE} =& R(1) - \log b \\
    R_{DW} =& \log \left(\frac{1}{T} \sum_{\tau=1}^T \exp(R(\tau))\right) - \log b
\end{align}
where $b$ and $T$ are the number of negative samples and context window size, respectively, and $\exp(\cdot)$ is the elementwise exponential. The proof of these facts is included in the Appendix. As we see, LINE preserves the PMI similarity at Markov time $1$, while DeepWalk factorizes the \emph{log-mean-exp}---a smooth approximation of the average---of the PMI matrices from Markov time $1$ to $T$.  

\subsubsection{Autocovariance} 
The autocovariance is defined as the covariance of $X_u(t)$ and $X_v(t+\tau)$ within a time interval $\tau$:
\begin{equation}
    \begin{aligned}
    R_{uv}(\tau) &= \cov(X_u(t), X_v(t+\tau)) \\
    &= \E \left[\left(X_u(t)-\E[X_u(t)]\right)\left(X_v(t+\tau)-\E[X_v(t+\tau)]\right)\right]
    \end{aligned}
    \label{eqn::autocovariance}
\end{equation}
The value of $R_{uv}(\tau)$ is a linear measure of the joint variability of the walk visiting probabilities for nodes $u$ and $v$ within time $\tau$.  
Similar to PMI, for an ergodic walk and starting probabilities $\pi$: 
\begin{equation}
    R_{uv}(\tau) = \pi_u  p(x(t{+}\tau){=}v|x(t){=}u) - \pi_u \pi_v
    \label{eqn::autocov_pair}
\end{equation}
where $R_{uv}(\tau)\in [-\pi_u\pi_v,\pi_u(1-\pi_v)]$. In the matrix form:
\begin{equation}
    R(\tau) = \Pi M^{\tau} - \pi \pi^T
    \label{eqn::autocov_matrix}
\end{equation}

The autocovariance (\autoref{eqn::autocov_matrix}) and PMI (\autoref{eqn::pmi_matrix}) matrices share a very similar form, differing only by the logarithm operation. However, this distinction has broad implications for graph embedding. PMI is a non-linear metric closely related to the \textit{sigmoid function}, which is a quite popular activation in deep learning, while autocovariance is related to a \textit{piecewise linear function}---see Section \ref{sec::undirected_sampling} for details.
We will compare PMI and autocovariance in multiple tasks in Section \ref{sec::experiments} and provide theoretical and empirical insights on their different performance in Section \ref{sec::insight}.

\subsection{Embedding Algorithm}
The goal of an embedding algorithm is to generate vector representations that preserve a given similarity metric:
\begin{equation}
\begin{aligned}
U^* &= \argmin_U \sum_{u,v} (\mathbf{u}_u^T\mathbf{u}_v^{} - R_{uv})^2\\
&= \argmin_U \norm{UU^T - R}_F^2
\end{aligned}
\label{eqn::obj_und_embed}
\end{equation}
where $\mathbf{u}^T_u\mathbf{u}_v^{}$ captures similarity in the embedding space, $R$ is a similarity matrix, and $\norm{\cdot}_F$ is the Frobenius norm. 

In the following sections, we discuss two different techniques to optimize this embedding objective.

\subsubsection{Matrix factorization}
Matrix factorization is an explicit way to optimize the embedding. 
It generates a low-rank approximation of $R$ as $UU^T$ with $\rank(UU^T) = d$. Because $R$ is symmetric, from the Eckart-Young-Mirsky theorem \cite{eckart1936approximation}, the optimal $U^*=Q_d\sqrt{\Lambda_d}$, where $R=Q\Lambda Q^T$ is the Singular Value Decomposition (SVD) of $R$. Notice that, different from classical spectral approaches \cite{chung1997spectral}, factorization-based embedding is not based on the graph Laplacian.

Direct SVD of $R$ has a complexity of $O(n^3)$, which is infeasible for most applications. However, scalable factorization is possible for sparse graphs. For autocovariance, one can apply the Lanczos method \cite{lehoucq1998arpack}, which only requires sparse matrix-vector multiplications. This approach reduces the complexity to $O(nd^2+md\tau)$. For PMI, as discussed in \cite{qiu2019netsmf}, one can construct the spectral sparsifier of the similarity matrix and apply Randomized SVD \cite{halko2011finding}. The resulting complexity for this method is $O(\widetilde{m}\tau \log n + \widetilde{m}d + nd^2 + d^3)$, where $\widetilde{m}=O(m\tau)$ is the number of non-zeros in the sparsifier. 

\subsubsection{Sampling}
\label{sec::undirected_sampling}
Sampling-based algorithms produce embeddings that implicitly optimize \autoref{eqn::obj_und_embed} by maximizing the likelihood of a corpus $\mathcal{D}$ generated based on samples from the process. A sample $i$ is a random-walk sequence $\langle v_1^{(i)}, v_2^{(i)}, \ldots v_L^{(i)}\rangle$ of length $L$ with the initial node $v_1^{(i)}$ selected according to a distribution $p(v)$---we will assume that $p(v)=\pi_v$ for the remainder of this section. From each sample, we extract pairs $(v_t^{(i)},v_{t+\tau}^{(i)})$, where $\tau$ is the Markov time scale. Thus, $\mathcal{D}$ is a multiset of node pairs $(u,v)$.

Different from matrix factorization, sampling algorithms produce two embedding matrices, $U$ and $V$, the source and target embeddings, respectively. The use of two embeddings was initially proposed by word2vec \cite{mikolov2013distributed} and is considered a better alternative. We will focus our discussion on algorithms that exploit \textit{negative sampling}, with $b$ negative samples, to efficiently generate embeddings. Let $z$ be a random variable associated with pairs $(u,v)$ such that $z=1$ if $(u,v)$ appears in $\mathcal{D}$ and $z=0$, otherwise. The log-likelihood $\ell$ of the corpus $\mathcal{D}$ can be expressed as:
\begin{equation}
\ell = \sum_{u,v}\!\#(u,v)(
\log(p(z\!=\!1|u,v))\!+\!b\mathbb{E}_{w\sim \pi}{}\log(p(z\!=\!0|u,w)))
\label{eqn::sampling_likelihood}
\end{equation}
where $\#(u,v)$ is the number of occurrences of the pair $(u,v)$ in $\mathcal{D}$.

The form of the conditional probability function $p(z|u,v)$ is determined by the similarity function. For instance, in the case of PMI, $p(z|u,v)$ is known to take the form of the \textit{sigmoid} function: $\sigma(x)=1/(1+\exp(x))$ \cite{qiu2018network}. More specifically, $p(z=1|u,v) = \sigma(\mathbf{u}_u^T\mathbf{v}^{}_v)$ and $p(z=0|u,v)  = \sigma(-\mathbf{u}_u^T\mathbf{v}_v^{})$. Here, we show how the objective in \autoref{eqn::sampling_likelihood} can be maximized for the autocovariance similarity. First, we define the following conditional probability function:
\begin{equation}
p(z=1|u,v) = \rho\left( \frac{\mathbf{u}_u^T\mathbf{v}_v+\pi_u\pi_v}{\mathbf{u}_u^T\mathbf{v}_v+(b+1)\pi_u\pi_v}\right)
\label{eqn::density_autocov_one}
\end{equation}

\begin{equation}
p(z=0|u,v) =  \rho\left(\frac{\pi_u\pi_v}{\mathbf{u}_u^T\mathbf{v}_v^{}+(b+1)\pi_u\pi_v}\right)
\label{eqn::density_autocov_zero}
\end{equation}
where $\rho(x)$ is a \textit{piecewise linear activation} with $\rho(x) = 0$, if $x < 0$, $\rho(x) = x,$ if $0 \leq x \leq 1$, and $\rho(x) = 1$, otherwise.

The following theorem formalizes the connection between the above defined conditional probability and autocovariance:
\begin{theorem}
For a large enough dimensionality $d$ ($d=\Omega(n)$), the embedding that maximizes \autoref{eqn::sampling_likelihood} with a conditional probability given by \autoref{eqn::density_autocov_one} and \autoref{eqn::density_autocov_zero} is such that:
\begin{equation}
      \mathbf{u}_u^T\mathbf{v}_v^{}= \frac{1}{b}\pi_u  p(x(t{+}\tau){=}v|x(t){=}u)) - \pi_u\pi_v
      \label{eqn::auto_cov_pair_neg}
\end{equation}
\label{thm::sampling_stab}
\end{theorem}
The proof of the theorem is provided in the Appendix. While here we only show the sampling-based algorithm for autocovariance, previous work has given a similar proof for PMI \cite{qiu2018network}.

We minimize the likelihood from \autoref{eqn::sampling_likelihood} using gradient descent. The time complexity of the sampling algorithms (for PMI and autocovariance) is $O(|\mathcal{D}|b)$, where $|\mathcal{D}|$ is the size of the corpus and $b$ is the number of negative samples---in practice, $|\mathcal{D}|=O(n)$.

\section{Experiments}
\label{sec::experiments}
In this section, we evaluate several random-walk based graph embedding methods defined according to our analytical framework on various downstream tasks and datasets. 
\begin{table}[htbp]
  \centering
  \caption{An overview of the datasets.}
    \begin{tabular}{cccc}
    \toprule
          & $|\mathcal{V}|$     & $|\mathcal{E}|$     & labels \\
    \midrule
    \textsc{BlogCatalog} & 10,312 & 333,983 & interests \\
    \textsc{Airport} & 3,158 & 18,606 & countries/continents \\
    \textsc{Wiki-words} & 4,777 & 92,157 & tags\\
    \textsc{PoliticalBlogs} & 1,222 & 16,717 & ideologies\\
    \textsc{Cora}  & 23,166 & 91,500 & categories/subcategories\\
    \textsc{Wiki-fields} & 10,675 & 137,606 & fields/subfields \\
    \bottomrule
    \end{tabular}%
  \label{tab::datasets}%
\end{table}%
\subsection{Dataset}
We apply six datasets (see \autoref{tab::datasets}) in our experiments.   
\begin{itemize}[leftmargin=*]
    \item \textsc{BlogCatalog} \cite{tang2009relational}: Undirected social network of bloggers with (multi) labels representing topics of interest.
    \item \textsc{Airport}: Flight network among global airports (from OpenFlights \cite{openflight}). Edges are weighted by the number of flights through the corresponding air routes. Labels represent the countries and continents where airports are located. The largest undirected connected component of the network is used in our experiments.
    \item \textsc{Wiki-words} \cite{mahoney2011large}. Undirected co-occurrence network of words in the first million bytes of the Wikipedia dump. Labels are Part-of-Speech (POS) tags inferred by the Stanford POS-Tagger. 
    \item \textsc{PoliticalBlogs} \cite{adamic2005political}. Network of hyperlinks between weblogs on US politics, recorded in 2005. Labels represent political ideologies (conservative vs liberal). The largest undirected connected component of the network is used in our experiments.
    \item \textsc{Cora} \cite{vsubelj2013model}: Directed citation network with categories (e.g., AI) and subcategories (e.g, Knowledge Representation) as labels. 
  \item \textsc{Wiki-fields}: Subset of the Wikipedia directed hyperlink network \cite{aspert2019graph} covering Computer Science, Mathematics, Physics and Biology for the year 2019. Fields and subfields are used as labels.
\end{itemize}
\subsection{Experiment Setting}
We evaluate embedding methods on three downstream tasks:
\begin{itemize}[leftmargin=*]
    \item Node classification. We follow the same procedure as \cite{perozzi2014deepwalk}. For each dataset, we randomly sample a proportion of node labels for training and the rest for testing. We use one-vs-rest logistic regression implemented by LIBLINEAR \cite{fan2008liblinear}
    for multi-class classification (\textsc{Airport} and \textsc{Cora}) and multi-label classification (\textsc{BlogCatalog}, \textsc{Wiki-words}, and \textsc{Wiki-fields}). To avoid the thresholding effect \cite{tang2009large} in the multi-label setting, we assume that the number of true labels for each node is known. Performance is reported as average \emph{Micro-F1} and \emph{Macro-F1} \cite{tsoumakas2009mining} for 10 repetitions. 
    \item Link prediction. We randomly remove $20\%$ of edges while ensuring that the residual graph is still connected and embed the residual graph. Edges are predicted as the top pairs of nodes ranked by either dot products of embeddings or classification scores from a logistic regression classifier with the concatenation of embeddings as input. We report \emph{precision@k} \cite{lu2011link} as the evaluation metric and also use \emph{recall@k} in our analysis, where $k$ is the number of top pairs, in terms of the ratio of removed edges. 
    \item Community detection. We use k-means---with k-means++ initialization \cite{arthur2006k}---for community detection, with the number of clusters set to be the actual number of communities. Normalized Mutual Information (\emph{NMI}) \cite{strehl2002cluster} between predicted and true communities is used as the evaluation metric. 
\end{itemize}
For all experiments, the number of embedding dimensions is set to $d=128$. When searching for the best Markov time, we sweep $\tau$ from 1 to 100. An implementation of our framework is available at \url{https://github.com/zexihuang/random-walk-embedding}.
\subsection{Results}
In this section, we evaluate how different components of the embedding methods affect their performance. 
\begin{figure*}[htbp]
  \centering
  \includegraphics[width=\textwidth]{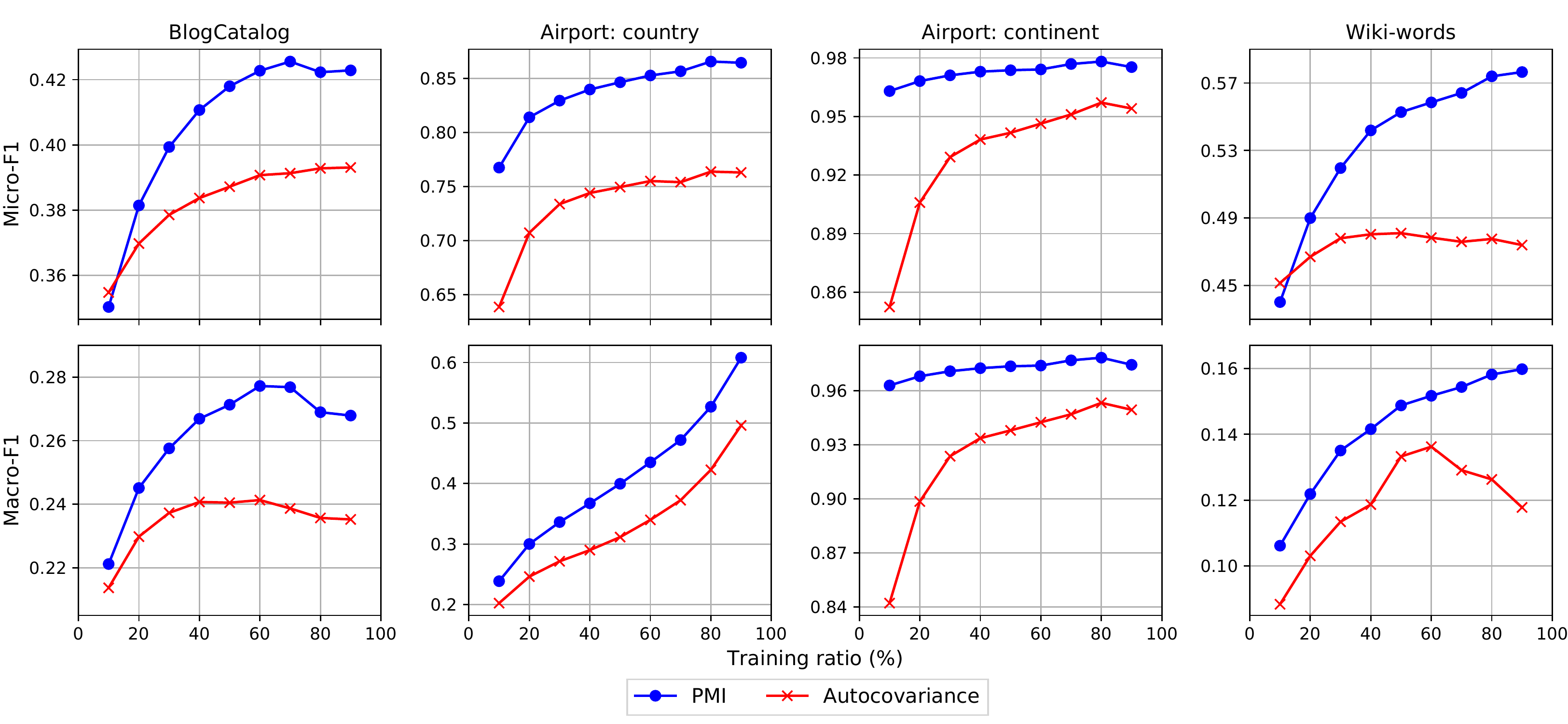}
  \caption{Node classification performance comparison between PMI and autocovariance on varying training ratios. PMI consistently outperforms autocovariance in all datasets. }
  \label{fig::ac_vs_pmi_classification}
\end{figure*}
\begin{figure*}[htbp]
  \centering
  \includegraphics[width=\textwidth]{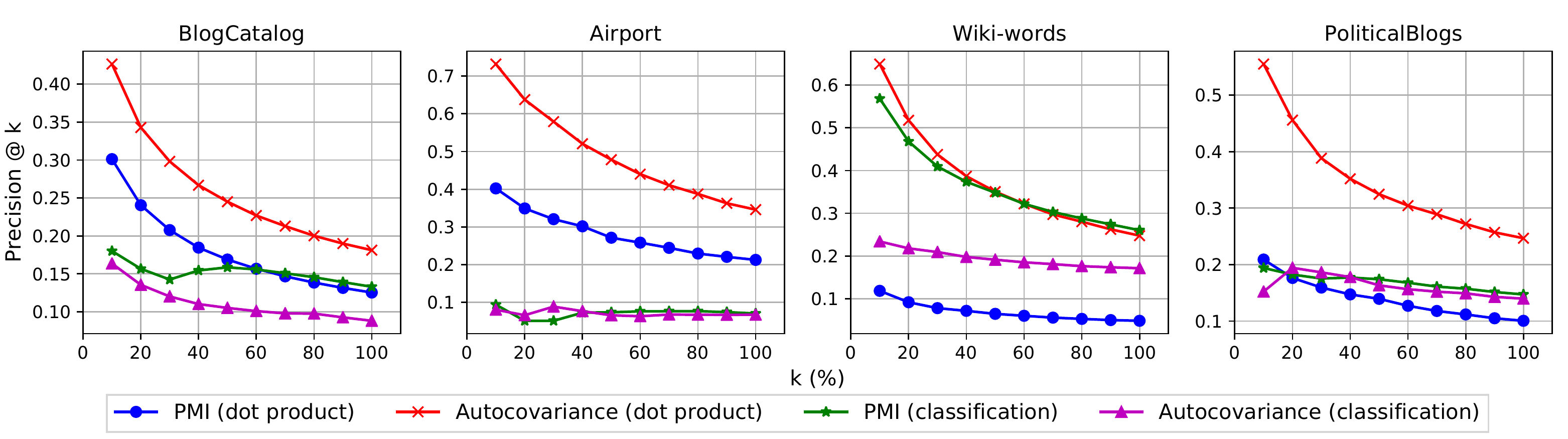}
  \caption{Link prediction performance comparison between PMI and autocovariance on varying percentages of top predictions ($k$). Autocovariance with dot product ranking consistently outperforms PMI (with either ranking scheme) in all datasets.}
  \label{fig::ac_vs_pmi_link_prediction}
\end{figure*}
\subsubsection{PMI vs autocovariance}
\label{subsubsec::pmi_vs_autocovariance}

We apply standard random-walks and the matrix factorization algorithm and analyze the difference between PMI and autocovariance similarity. 

\autoref{fig::ac_vs_pmi_classification} shows the node classification performance on undirected datasets (excluding \textsc{PoliticalBlog} as it is a simple binary classification task). We select the Markov time $\tau$ with the best performance for the 50\% training ratio. Results show that PMI consistently outperforms autocovariance for both \emph{Micro-F1}/\emph{Macro-F1} scores. The average gain is 6.0\%/11.3\% for \textsc{BlogCatalog}, 14.2\%/24.5\% and 4.6\%/5.2\% for \textsc{Airport} with country and continent labels, and 12.9\%/20.0\% for \textsc{Wiki-words}. This is a piece of evidence that PMI, which is non-linear, is more effective at node-level tasks.

\autoref{fig::ac_vs_pmi_link_prediction} shows the results for link prediction, where we select the best Markov time for $k=100\%$. Autocovariance with dot product ranking consistently outperforms PMI with either ranking scheme in all datasets. The average gains over the best ones are 44.1\% in \textsc{BlogCatalog}, 72.9\% in \textsc{Airport}, 2.2\% in \textsc{Wiki-words}, and 101.4\% in \textsc{PoliticalBlogs}. To the best of our knowledge, we are the first to observe the effectiveness of autocovariance on edge-level tasks.

While both dot product-based and classification-based link prediction have been widely used in the embedding literature, our results show that dot product is a clearly superior choice for autocovariance embedding. It also leads to better or similar performance for PMI embedding for all but the \textsc{Wiki-words} dataset. Thus, we will only show results based on dot products in later experiments.  

Results in this section beg the deeper question of why specific similarities and ranking schemes lead to better performance. We will provide theoretical and empirical insights on this in Section \ref{sec::insight}.  

\begin{figure*}[htbp]
  \centering
  \includegraphics[width=\textwidth]{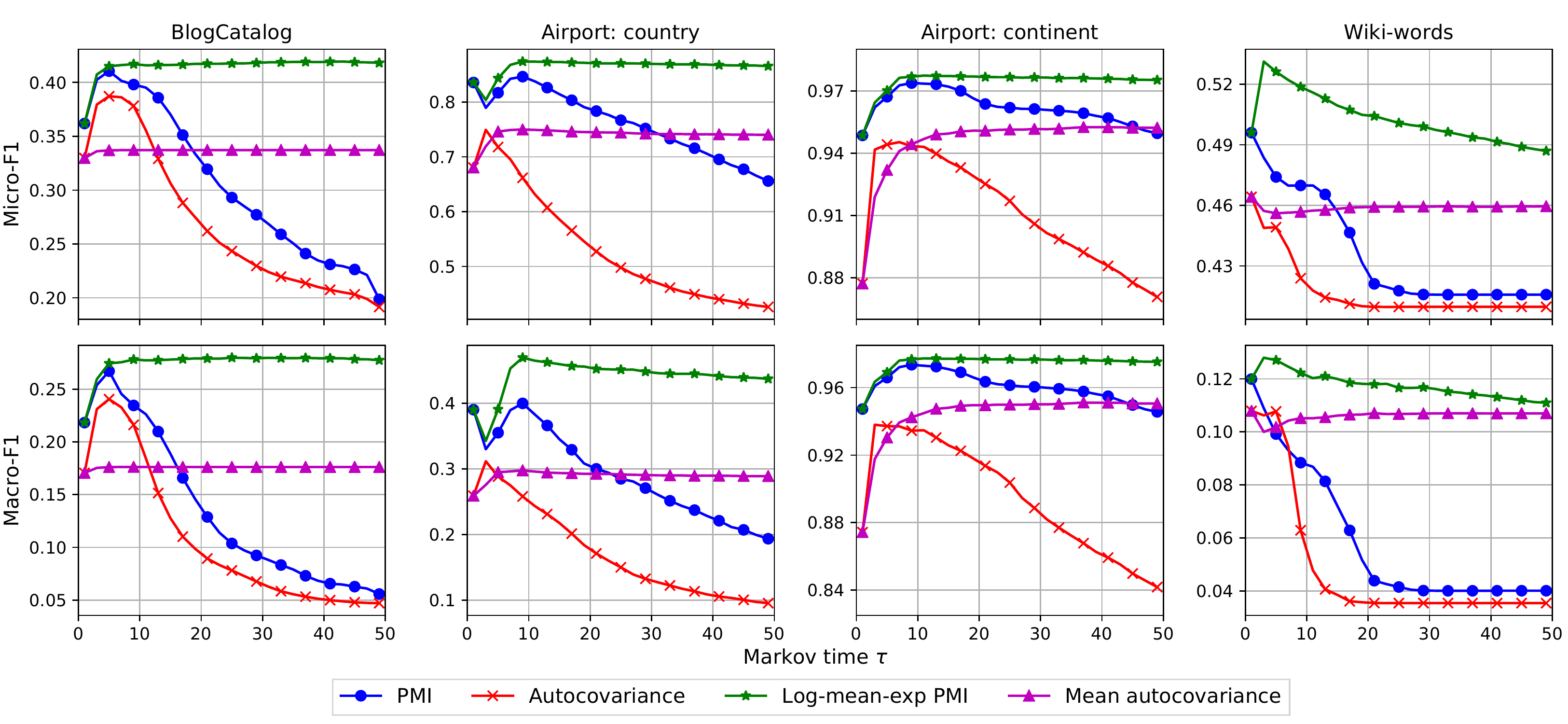}
  \caption{Node classification results for PMI, autocovariance, and their moving means on varying Markov times. While both means stabilize the performance for large Markov times, only \emph{log-mean-exp} PMI consistently increases the peak performance.}
  \label{fig::multiscale_classification}
\end{figure*}
\begin{figure}[htbp]
  \centering
  \includegraphics[width=\columnwidth]{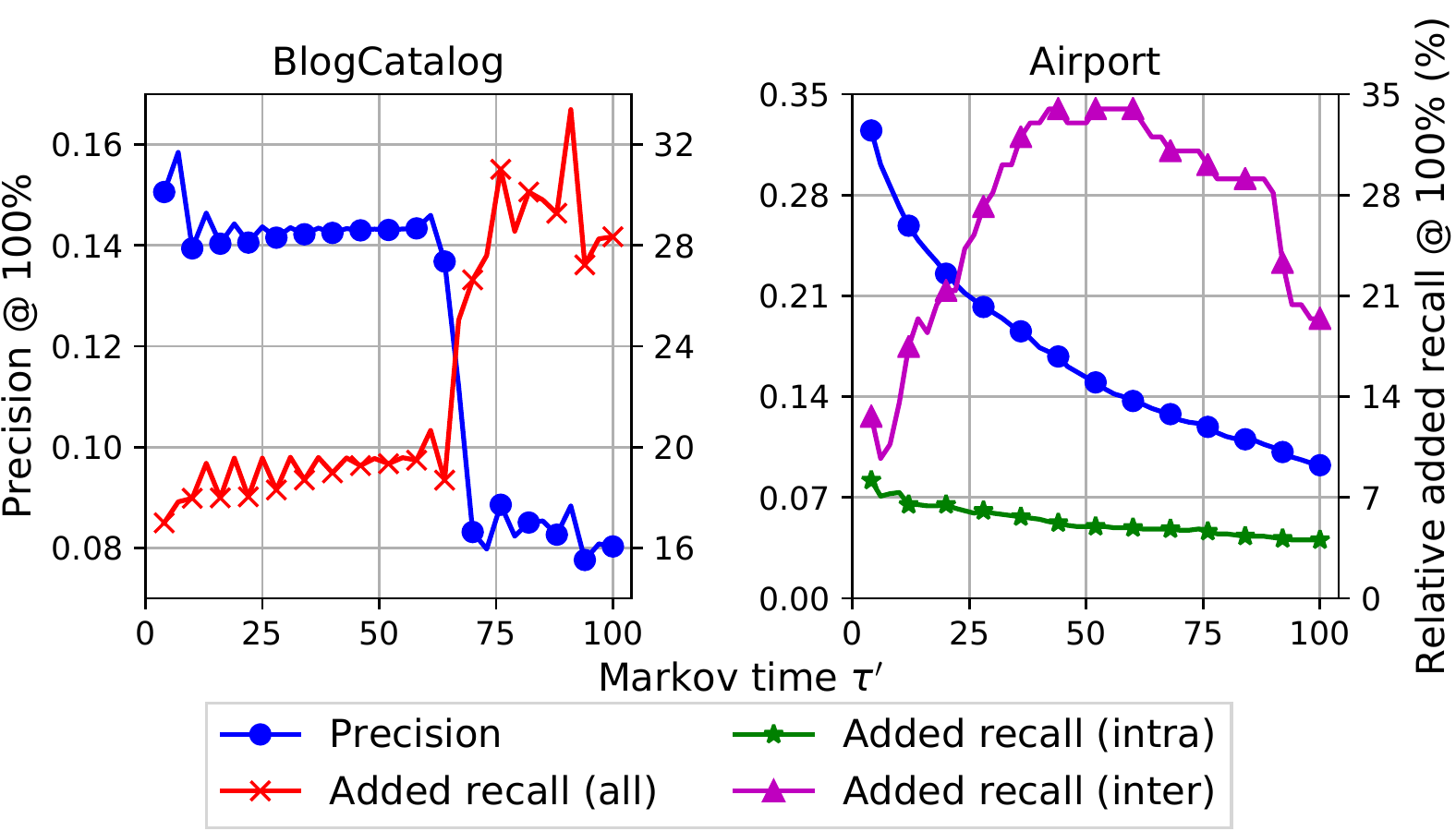}
  \caption{Relative added recall for link prediction after adding predictions from a larger Markov time $\tau'$ to the best Markov time $\tau^*$. While the precision drops at larger Markov times, they predict a distinct set of true (inter-community) edges from those revealed at the best (and small) time.}
  \label{fig::multiscale_link_prediction}
\end{figure}
\subsubsection{Multiscale}
\label{subsubsec::multiscale}

In this section, we analyze the effect of different Markov time scales on the embedding performance using standard random-walks and matrix factorization.

\autoref{fig::multiscale_classification} compares the node classification performance for different similarity measures on varying Markov time scales from 1 to 50. The training ratio is fixed at 50\% for all datasets. The metrics have a peak performance at a certain Markov time. 
However, notice that this aspect has not been given much relevance by previous work on embedding. This is a legacy of DeepWalk \cite{perozzi2014deepwalk}, which, as we have shown, implicitly applies the \emph{log-mean-exp} of the PMI matrix within a range of Markov times. Thus, we also show the performance for the moving \emph{log-mean-exp} PMI and mean autocovariance (it is linear) in \autoref{fig::multiscale_classification}. Interestingly, while both mean versions have smoother results, only \emph{log-mean-exp} PMI has a higher peak performance---with gains of 2.2\%/4.8\% for \textsc{BlogCatalog}, 3.3\%/17.4\% and 0.4\%/0.4\% for \textsc{Airport} with country and continent labels, and 7.1\%/6.7\% for \textsc{Wiki-words}. This shows \emph{log-mean-exp} is indeed an effective approach to combine PMI similarity at different scales. Conversely, we cannot apply a similar strategy for autocovariance.

We observe similar results for community detection (see \autoref{fig::multiscale_community_detection} in the Appendix). Also included are results for Markov Stability \cite{delvenne2010stability}, which applies clustered autocovariance for multiscale community detection. For both countries and continents in \textsc{Airport}, (\emph{log-mean-exp}) PMI achieves the best performance. This is another piece of evidence for the effectiveness of PMI on node-level tasks.

We then evaluate the effect of different Markov time scales on link prediction using \textsc{BlogCatalog} and \textsc{Airport}. We first note that a moving mean does not improve the performance for either PMI or autocovariance (see \autoref{fig::multiscale_link_prediction_bad} in the Appendix). We hypothesize the reason to be that each edge plays a structural role that is specific to a few relevant scales (e.g., connecting two \emph{mid-sized} communities).  
To validate it, we first find the best Markov time $\tau^*$ in terms of \emph{precision@100\%} for autocovariance. Then, for every Markov time $\tau'$ larger than $\tau^*$, we compute its added \emph{recall@100\%}---i.e., the proportion of correctly predicted edges at $\tau'$ that are not predicted at $\tau^*$. We show the relative gain of added recall compared to $\tau^*$ and the $precision@100\%$ for different values of $\tau'$ in \autoref{fig::multiscale_link_prediction}. For \textsc{BlogCatalog}, while precision drops as $\tau'$ increases, the relative added recall increases to up to 33.4\% at $\tau' = 91$ (4,042 new edges added to the 12,104 edges correctly predicted at $\tau^*=3$). 
To further understand the roles of those edges, we show the relative added recall for intra-continent edges and inter-continent edges separately for \textsc{Airport}. Most of the gain is for inter-community edges (up to 34.0\% at $\tau'=42$).
This observation suggests that link prediction can be improved by accounting for the scale of edges.

\subsubsection{Factorization vs sampling}
\label{subsubsec::factorization_vs_sampling}
We now switch our focus to evaluating the performance of sampling and matrix factorization algorithms. For PMI, we apply a publicly available node2vec implementation\footnote{\url{https://github.com/eliorc/node2vec/blob/master/node2vec/node2vec.py}} with parameters that make it equivalent to DeepWalk. 
For autocovariance, we implement our algorithm using the PyTorch framework with Adam optimizer \cite{kingma2014adam}. Parameters are set based on \cite{grover2016node2vec}: 10 walks per node with length 80, 400 epochs for convergence, 1000 walks per batch, and 5 negative samples.

\autoref{fig::factorization_vs_sampling_link_prediction} shows the link prediction performance where the context window size for both PMI and autocovariance is set to the best Markov time for \emph{precision@100\%}. We focus on link prediction because \cite{qiu2018network} has already shown evidence that matrix factorization is superior to sampling for node classification. 
Factorization achieves better performance for both datasets and similarity functions. The average gain of \emph{precision@k} for PMI is 277.0\%/553.7\% on \textsc{BlogCatalog}/\textsc{Airport}, and 240.7\%/121.5\% for autocovariance.

\autoref{fig::correlation} (in the Appendix) shows how dot products of 16-D embeddings for Zachary's karate club generated using sampling and matrix factorization algorithms approximate the entries of similarity matrices. We notice that while embeddings produced by both sampling and factorization approaches are correlated with the similarities, matrix factorization achieves a higher correlation.

\begin{figure}[htbp]
  \centering
  \includegraphics[width=\columnwidth]{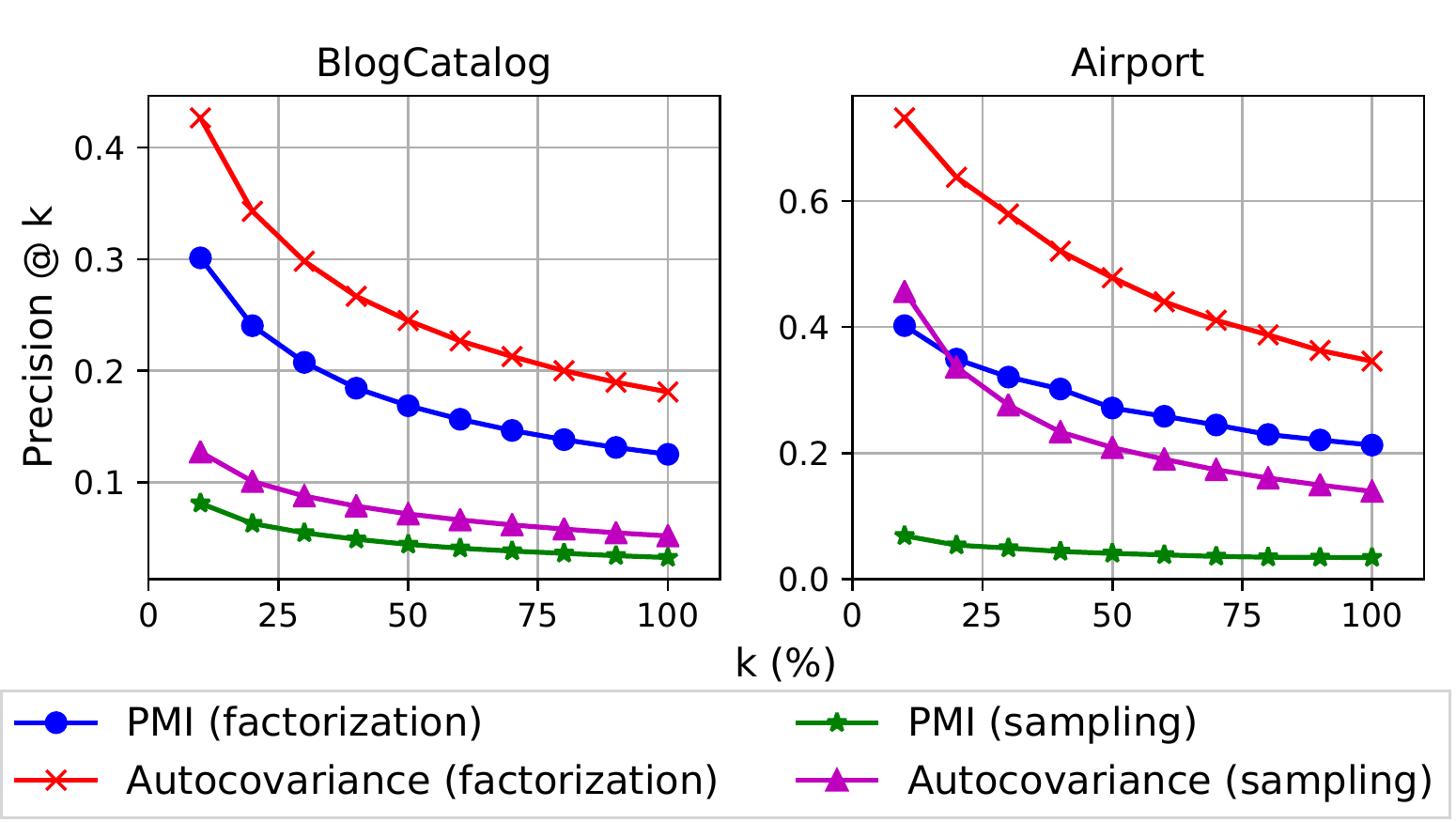}
  \caption{Link prediction performance for PMI and autocovariance using matrix factorization and sampling algorithms on varying percentages of top predictions ($k$). Factorization algorithms achieve the best performance.}
  \label{fig::factorization_vs_sampling_link_prediction}
\end{figure}
\subsubsection{Directed vs undirected}
\label{subsubsec::directed_vs_undirected}
The final part of our evaluation compares embeddings for undirected and directed graphs, which requires different random-walk processes, as discussed in Section \ref{subsec::random-walk}. The results for \textsc{Cora} and \textsc{Wiki-fields} are shown in Figures \ref{fig::directed_vs_undirected_link_prediction} and \ref{fig::directed_vs_undirected_classification} (in the Appendix). Overall, directed embeddings outperform undirected ones for the very top ranked pairs in link prediction, while undirected embeddings are better in node classification.

\section{Insights on PMI vs Autocovariance}
\label{sec::insight}

This section provides both theory and experiments supporting the differences in performance achieved by autocovariance and PMI. We will focus on link prediction, for which autocovariance was shown to achieve the best performance.
Our analysis might benefit the design of novel embedding methods using our framework.

We will assume that the graph $G$ has community structure \cite{girvan2002community} and heterogeneous degree distribution \cite{barabasi2003scale}, which hold for many real graphs and for generalizations of Stochastic Block Models (SBM) \cite{karrer2011stochastic}. Moreover, let the number of dimensions $d$ be large enough to approximate the similarities well. For simplicity, we consider an SBM instance with two clusters and intra-cluster edge probability $p$ significantly higher than the inter-cluster probability $q$.

We will use the dot product setting for link prediction \cite{seshadhri2020impossibility}.
Specifically, we estimate the probability of an edge as $P(e_{u,v}) \propto \max(0,\mathbf{u}^T\mathbf{v})$, where we have conveniently dropped the embedding subscripts. Further, for $p \gg q$, $\mathbf{u}^T\mathbf{v} > 0$ iff $C(u)=C(v)$, where $C(v)$ is the cluster $v$ belongs to. We then have the following observations.

\begin{observation}
Link prediction based on dot products correlates predicted node degrees 
and the 2-norms of embedding vectors:
$$\widetilde{\deg}(u) \approx \sum_{v\in C(u)-\{u\}}\mathbf{u}^T\mathbf{v}=\norm{\mathbf{u}}_2\sum_{v\in C(u)-\{u\}}\norm{\mathbf{v}}_2\cos(\theta_{\mathbf{u},\mathbf{v}})$$ where $\cos(\theta_{u,v})$ is the cosine of the angle between $\mathbf{u}$ and $\mathbf{v}$. 
\end{observation}

The above property follows from the community structure, as the majority of the edges will be within communities. We now look at the norms of vectors generated by autocovariance and PMI. 

\begin{observation}
For autocovariance similarity:
$$\norm{\mathbf{u}}_2^2 = \pi(u)([M^{\tau}]_{u,u}-\pi(u))=\frac{\deg(u)}{2m}\left([M^{\tau}]_{u,u}-\frac{\deg(u)}{2m}\right)$$
\end{observation}
Norms for autocovariance depend on two factors. The first is proportional to the actual node degree, while the second expresses whether $u$ belongs to a community at the scale $\tau$. For SBM, there exists a $\tau$ such that $[M^{\tau}]_{u,u}$ is close to $\deg(u)/m$ in expectation. That is the case when the process is almost stationary within $C(u)$, with $m/2$ expected edges, but not in the entire graph. It implies that the embedding norms are proportional to actual node degrees. 
Combining this with Observation 1, we can conclude that autocovariance embedding predicts degrees related to the actual ones.

\begin{observation}
For PMI similarity:
$$\norm{\mathbf{u}}_2^2 = \log\left(\pi(u)[M^{\tau}]_{u,u}\right)-\log\left(\pi^2(u)\right)=\log\left(\frac{2m[M^{\tau}]_{u,u}}{\deg(u)}\right)$$
\end{observation}

Notice that as $[M^{\tau}]_{u,u}$ approaches to $\deg(u)/m$, $\norm{\mathbf{u}}_2$ becomes constant. As a consequence, different from autocovariance, norms for PMI embeddings are not directly related with node degrees. 

In \autoref{fig::pi_norm_correlation}, we provide empirical evidence for Observations 2 and 3---the correlation between actual degrees and embedding norms for \textsc{BlogCatalog} and \textsc{Airport}. 
The correlation for autocovariance is significantly higher than that for PMI, showing the ability of autocovariance to capture heterogeneous degree distributions.

\autoref{fig::pred_toy_graph} shows the predicted links for an instance of SBM ($p=0.15$, $q=0.02$) with two hubs---generated by merging 20 nodes inside each community. Visualization is based on t-SNE \cite{hinton2003stochastic} projection from 16-D embeddings. 
Around half of the edges (97 out of top 200) are predicted to be connected to the hubs using autocovariance, but none using PMI. 
This example shows how autocovariance with dot product ranking enables state-of-the-art link prediction.  

\begin{figure}
    \centering
    \includegraphics[width=0.45\textwidth]{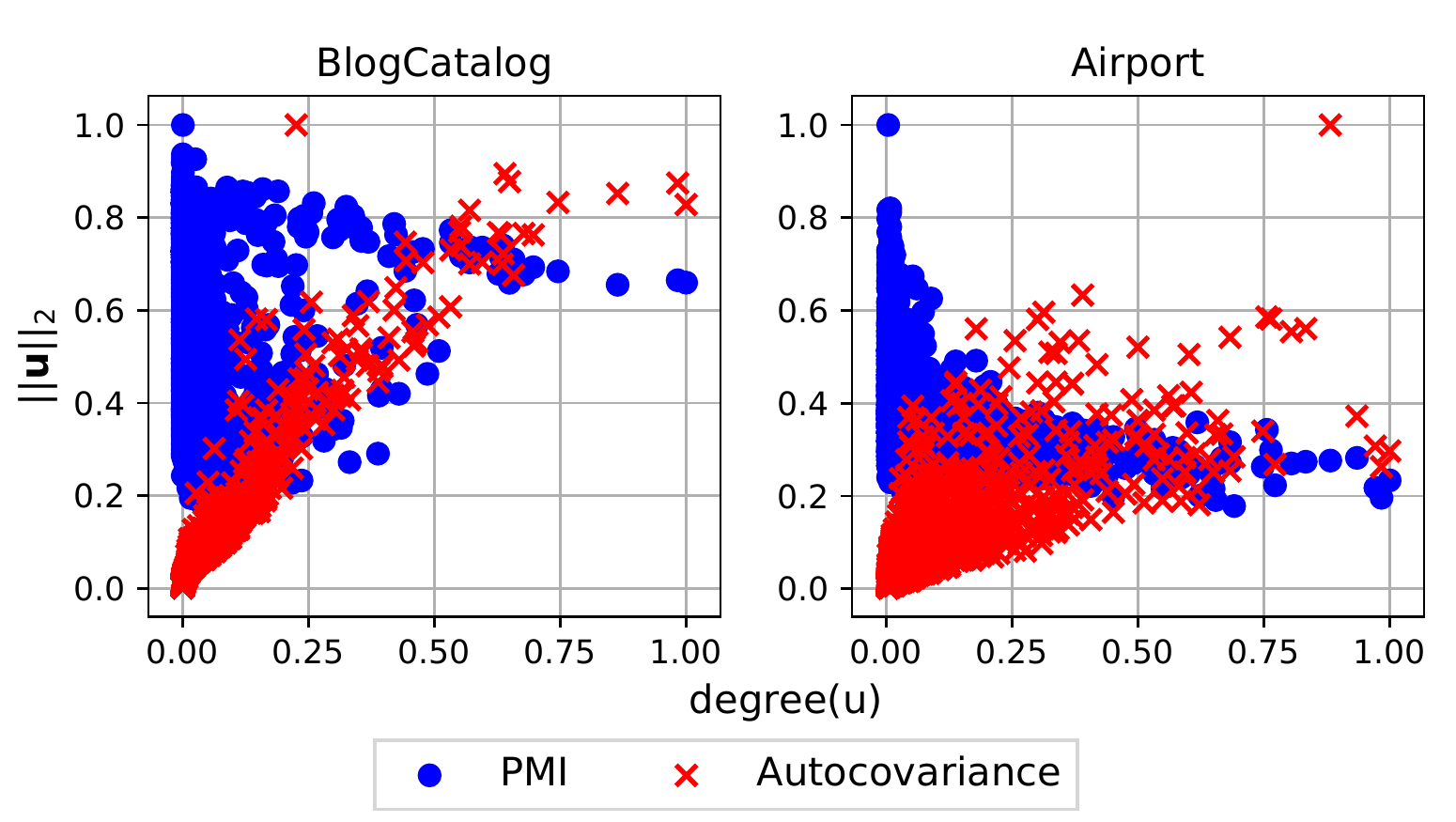}
    \caption{Correlation between (max normalized) degrees and 2-norms of embedding vectors based on PMI and autocovariance. Autocovariance produces embeddings with norms that are well correlated with node degrees in both datasets. \label{fig::pi_norm_correlation}}
\end{figure}
\begin{figure}
    \centering
    \subfloat[PMI]{ \includegraphics[width=0.22\textwidth]{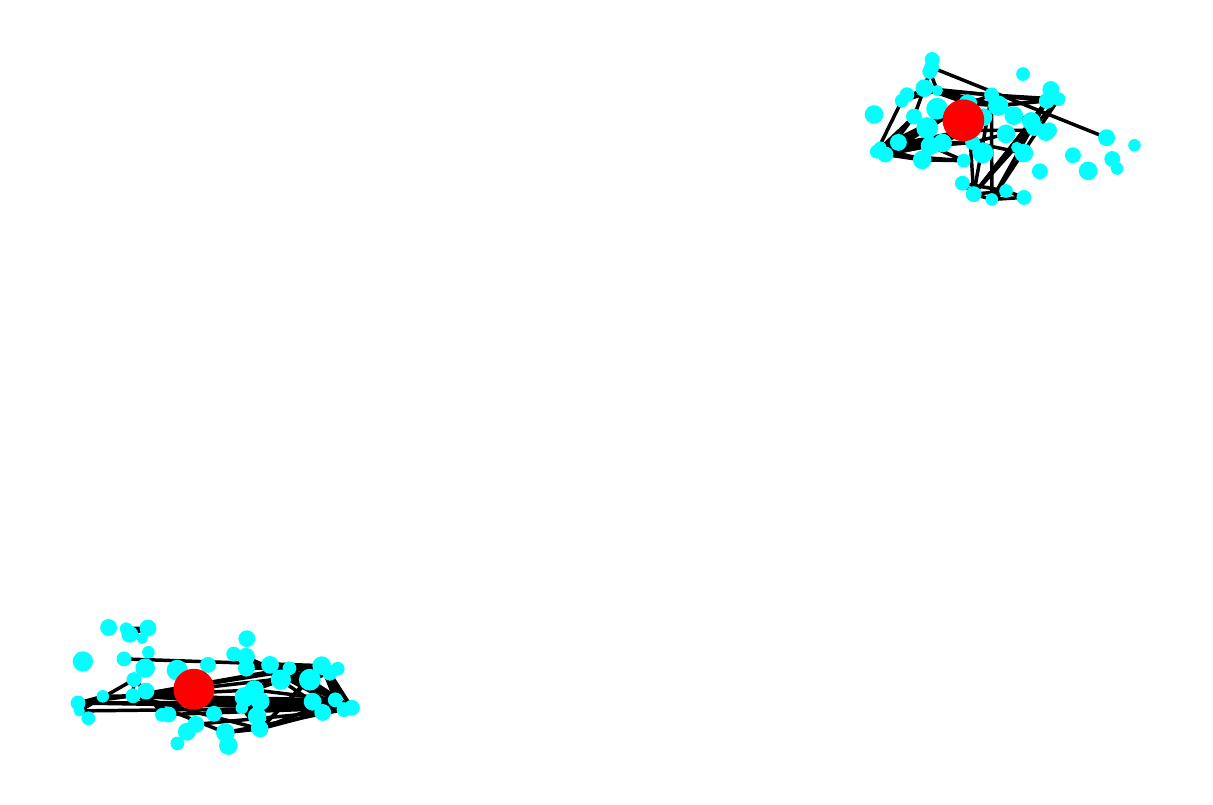}}
    \subfloat[Autocovariance]{ \includegraphics[width=0.22\textwidth]{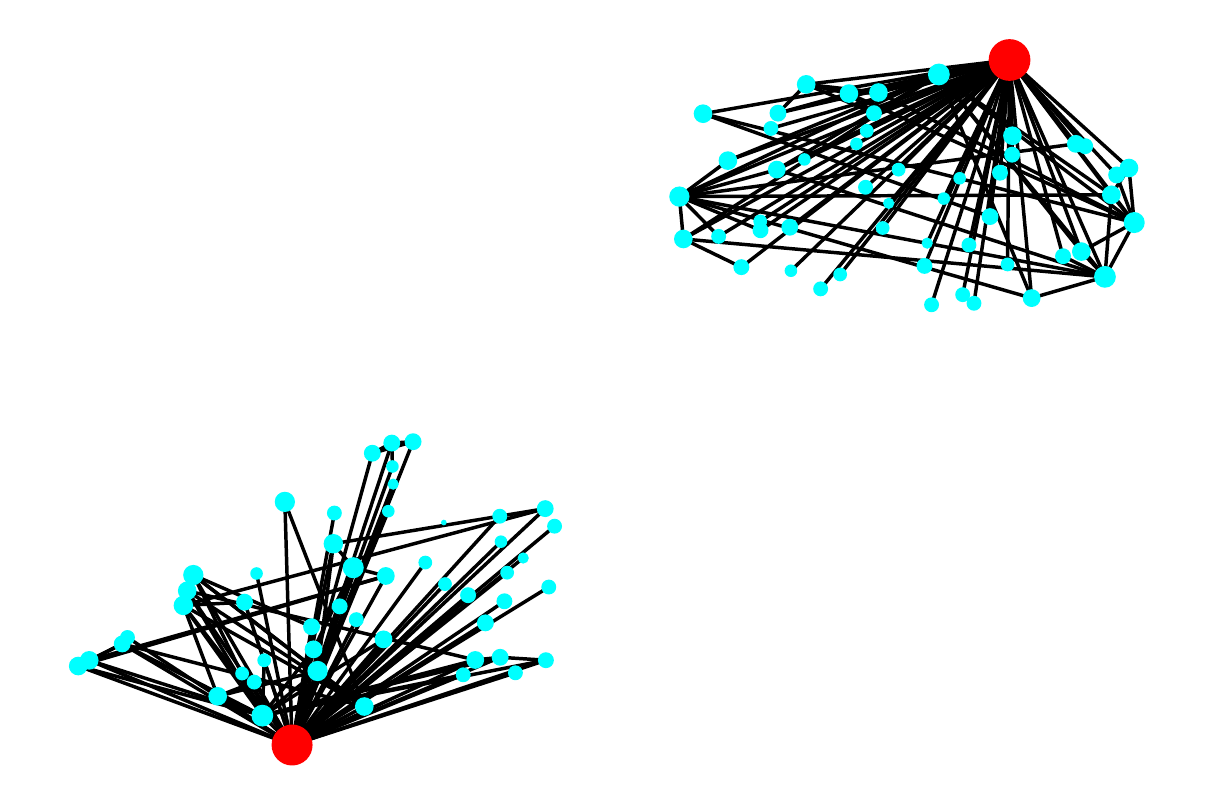}}
    \caption{Edges predicted using PMI and autocovariance embeddings for a synthetic graph with community structure and hubs (one per community, shown in red). Autocovariance is more effective at capturing the hubs in the graph. \label{fig::pred_toy_graph}}
\end{figure}
The analysis presented here does not apply to node-level tasks (node classification and clustering), where the degree distribution does not play a major role. In fact, \autoref{fig::pred_toy_graph} shows that PMI produces a better separation between clusters with hubs for each cluster positioned in the middle, which is desired for node-level tasks. This property might be explained by the non-linearity of PMI. 

\section{Related Work}
\label{sec::related_work}
We refer to  \cite{hamilton2017representation,cui2018survey,cai2018comprehensive,goyal2018graph,chami2020machine} for an overview of graph embedding and representation learning. 
The general problem of embedding vertices of a graph into a vector space can be traced back to much earlier work on embedding metric spaces \cite{bourgain1985lipschitz}, spectral graph theory \cite{chung1997spectral}, nonlinear dimensionality reduction \cite{belkin2002laplacian} and graph drawing \cite{diaz2002survey}. However, the more recent appeal for graph embedding methods coincided with the renewed interest in deep learning, specially the skip-gram model for text, such as word2vec \cite{mikolov2013distributed}. DeepWalk \cite{perozzi2014deepwalk}, which introduced the idea of using random-walks as graph counterparts of sentences in skip-gram, is considered the pioneering work on random-walk based graph embedding. Subsequent extensions have led to an extensive literature \cite{grover2016node2vec,tang2015line,ou2016asymmetric,zhou2017scalable}. Random-walk based embedding models have also been proposed for heterogeneous \cite{dong2017metapath2vec}, dynamic \cite{du2018dynamic} and multilayer networks \cite{zhang2018scalable}.

Previous work has proposed embeddings based on specific processes. For instance, node2vec \cite{grover2016node2vec} applies a biased random-walk process that is flexible enough to capture homophily and structural similarity. In \cite{zhou2017scalable}, the authors propose a variation of DeepWalk with PageRank \cite{page1999pagerank} instead of the standard random-walks. These studies motivate generalizing embedding to other processes, such as the Ruelle-Bowen random-walk, which has maximum entropy \cite{ruelle2004thermodynamic}, and even continuous-time random-walks \cite{nguyen2018continuous,lambiotte2014random}.

With the exception of \cite{schaub2019multiscale}, random-walk based embeddings cannot generalize---either explicitly or implicitly---to other similarities besides Pointwise Mutual Information (PMI) \cite{qiu2018network}. As an alternative, autocovariance similarity is the basis of Markov Stability \cite{delvenne2010stability,delvenne2013stability,lambiotte2014random}, an effective multiscale community detection algorithm validated in many real networks. In \cite{schaub2019multiscale}, the authors introduced a general embedding scheme based on control theory, of which autocovariance is a special case. However, they neither contextualized it with other existing embedding approaches nor provided any theoretical or empirical evaluation on any particular task. A link prediction algorithm based on autocovariance was introduced in \cite{gao2019link}, but it does not produce embeddings. Integrating autocovariance into the broader framework of random-walk embedding---with different processes and embedding algorithms---and investigating its properties both theoretically and empirically is a contribution of our work.

For some time, graph embedding models were divided into those based on skip-gram \cite{perozzi2014deepwalk,grover2016node2vec,tang2015line}, which are trained using sampled walks and stochastic gradient descent (or its variations) and those based on matrix factorization \cite{ou2016asymmetric,cao2015grarep}. It was later shown that, under certain assumptions, skip-gram based text embedding models perform matrix factorization implicitly \cite{levy2014neural}. 
More recently, Qiu et al. \cite{qiu2018network} showed that DeepWalk, node2vec and LINE \cite{tang2015line} also perform implicit matrix factorization. Our framework incorporates both sampling-based and matrix factorization algorithms and we show how the former can generalize to similarities beyond PMI.

Multiscale graph embedding based on random-walks has attracted limited interest in the literature \cite{chen2017harp,xin2019marc}, though many real-life networks are known to have a multiscale (or hierarchical) structure \cite{ravasz2003hierarchical,clauset2008hierarchical}. Notice that we focus on graphs where scales are defined by clusters/communities. This is different from graphs where vertices belong to different levels of a hierarchy, such as trees, for which hyperbolic embedding \cite{nickel2017poincare} is a promising approach.

\section{Conclusion}
\label{sec::conclusions}

We have introduced a framework that provides a renewed bearing on random-walk based graph embedding. This area has capitalized on the close connection with skip-gram but has also been biased by some of its design choices. Our framework has enabled us to scrutinize them, which will benefit researchers and practitioners alike. 
In the future, we want to explore alternative choices of components in our framework, as well as extending it to graphs with richer information, such as signed, dynamic, and attributed graphs.

\begin{acks}
This work is partially funded by NSF via grant IIS 1817046 and DTRA via grant HDTRA1-19-1-0017.
\end{acks}

%

\clearpage
\bibliographystyle{ACM-Reference-Format}
\bibliography{reference}


\begin{thebibliography}{66}


\ifx \showCODEN    \undefined \def \showCODEN     #1{\unskip}     \fi
\ifx \showDOI      \undefined \def \showDOI       #1{#1}\fi
\ifx \showISBNx    \undefined \def \showISBNx     #1{\unskip}     \fi
\ifx \showISBNxiii \undefined \def \showISBNxiii  #1{\unskip}     \fi
\ifx \showISSN     \undefined \def \showISSN      #1{\unskip}     \fi
\ifx \showLCCN     \undefined \def \showLCCN      #1{\unskip}     \fi
\ifx \shownote     \undefined \def \shownote      #1{#1}          \fi
\ifx \showarticletitle \undefined \def \showarticletitle #1{#1}   \fi
\ifx \showURL      \undefined \def \showURL       {\relax}        \fi
\providecommand\bibfield[2]{#2}
\providecommand\bibinfo[2]{#2}
\providecommand\natexlab[1]{#1}
\providecommand\showeprint[2][]{arXiv:#2}

\bibitem[\protect\citeauthoryear{Adamic and Glance}{Adamic and Glance}{2005}]%
        {adamic2005political}
\bibfield{author}{\bibinfo{person}{Lada~A Adamic} {and}
  \bibinfo{person}{Natalie Glance}.} \bibinfo{year}{2005}\natexlab{}.
\newblock \showarticletitle{The political blogosphere and the 2004 US election:
  divided they blog}. In \bibinfo{booktitle}{\emph{Workshop on Link
  discovery}}.
\newblock


\bibitem[\protect\citeauthoryear{Arthur and Vassilvitskii}{Arthur and
  Vassilvitskii}{2006}]%
        {arthur2006k}
\bibfield{author}{\bibinfo{person}{David Arthur} {and} \bibinfo{person}{Sergei
  Vassilvitskii}.} \bibinfo{year}{2006}\natexlab{}.
\newblock \bibinfo{booktitle}{\emph{k-means++: The advantages of careful
  seeding}}.
\newblock \bibinfo{type}{{T}echnical {R}eport}.
  \bibinfo{institution}{Stanford}.
\newblock


\bibitem[\protect\citeauthoryear{Aspert, Miz, Ricaud, and Vandergheynst}{Aspert
  et~al\mbox{.}}{2019}]%
        {aspert2019graph}
\bibfield{author}{\bibinfo{person}{Nicolas Aspert}, \bibinfo{person}{Volodymyr
  Miz}, \bibinfo{person}{Benjamin Ricaud}, {and} \bibinfo{person}{Pierre
  Vandergheynst}.} \bibinfo{year}{2019}\natexlab{}.
\newblock \showarticletitle{A graph-structured dataset for Wikipedia research}.
  In \bibinfo{booktitle}{\emph{WebConf}}.
\newblock


\bibitem[\protect\citeauthoryear{Barab{\'a}si and Bonabeau}{Barab{\'a}si and
  Bonabeau}{2003}]%
        {barabasi2003scale}
\bibfield{author}{\bibinfo{person}{Albert-L{\'a}szl{\'o} Barab{\'a}si} {and}
  \bibinfo{person}{Eric Bonabeau}.} \bibinfo{year}{2003}\natexlab{}.
\newblock \showarticletitle{Scale-free networks}.
\newblock \bibinfo{journal}{\emph{Scientific american}} \bibinfo{volume}{288},
  \bibinfo{number}{5} (\bibinfo{year}{2003}), \bibinfo{pages}{60--69}.
\newblock


\bibitem[\protect\citeauthoryear{Belkin and Niyogi}{Belkin and Niyogi}{2002}]%
        {belkin2002laplacian}
\bibfield{author}{\bibinfo{person}{Mikhail Belkin} {and}
  \bibinfo{person}{Partha Niyogi}.} \bibinfo{year}{2002}\natexlab{}.
\newblock \showarticletitle{Laplacian eigenmaps and spectral techniques for
  embedding and clustering}. In \bibinfo{booktitle}{\emph{NeurIPS}}.
\newblock


\bibitem[\protect\citeauthoryear{Bourgain}{Bourgain}{1985}]%
        {bourgain1985lipschitz}
\bibfield{author}{\bibinfo{person}{Jean Bourgain}.}
  \bibinfo{year}{1985}\natexlab{}.
\newblock \showarticletitle{On Lipschitz embedding of finite metric spaces in
  Hilbert space}.
\newblock \bibinfo{journal}{\emph{Israel Journal of Mathematics}}
  \bibinfo{volume}{52}, \bibinfo{number}{1-2} (\bibinfo{year}{1985}),
  \bibinfo{pages}{46--52}.
\newblock


\bibitem[\protect\citeauthoryear{Cai, Zheng, and Chang}{Cai
  et~al\mbox{.}}{2018}]%
        {cai2018comprehensive}
\bibfield{author}{\bibinfo{person}{Hongyun Cai}, \bibinfo{person}{Vincent~W
  Zheng}, {and} \bibinfo{person}{Kevin Chen-Chuan Chang}.}
  \bibinfo{year}{2018}\natexlab{}.
\newblock \showarticletitle{A comprehensive survey of graph embedding:
  Problems, techniques, and applications}.
\newblock \bibinfo{journal}{\emph{IEEE TKDE}} \bibinfo{volume}{30},
  \bibinfo{number}{9} (\bibinfo{year}{2018}), \bibinfo{pages}{1616--1637}.
\newblock


\bibitem[\protect\citeauthoryear{Cao, Lu, and Xu}{Cao et~al\mbox{.}}{2015}]%
        {cao2015grarep}
\bibfield{author}{\bibinfo{person}{Shaosheng Cao}, \bibinfo{person}{Wei Lu},
  {and} \bibinfo{person}{Qiongkai Xu}.} \bibinfo{year}{2015}\natexlab{}.
\newblock \showarticletitle{Grarep: Learning graph representations with global
  structural information}. In \bibinfo{booktitle}{\emph{CIKM}}.
\newblock


\bibitem[\protect\citeauthoryear{Chami, Abu-El-Haija, Perozzi, R{\'e}, and
  Murphy}{Chami et~al\mbox{.}}{2020}]%
        {chami2020machine}
\bibfield{author}{\bibinfo{person}{Ines Chami}, \bibinfo{person}{Sami
  Abu-El-Haija}, \bibinfo{person}{Bryan Perozzi}, \bibinfo{person}{Christopher
  R{\'e}}, {and} \bibinfo{person}{Kevin Murphy}.}
  \bibinfo{year}{2020}\natexlab{}.
\newblock \showarticletitle{Machine Learning on Graphs: A Model and
  Comprehensive Taxonomy}.
\newblock \bibinfo{journal}{\emph{arXiv:2005.03675}} (\bibinfo{year}{2020}).
\newblock


\bibitem[\protect\citeauthoryear{Chanpuriya and Musco}{Chanpuriya and
  Musco}{2020}]%
        {chanpuriya2020infinitewalk}
\bibfield{author}{\bibinfo{person}{Sudhanshu Chanpuriya} {and}
  \bibinfo{person}{Cameron Musco}.} \bibinfo{year}{2020}\natexlab{}.
\newblock \showarticletitle{Infinitewalk: Deep network embeddings as Laplacian
  embeddings with a nonlinearity}. In \bibinfo{booktitle}{\emph{SIGKDD}}.
\newblock


\bibitem[\protect\citeauthoryear{Chen, Perozzi, Hu, and Skiena}{Chen
  et~al\mbox{.}}{2018}]%
        {chen2017harp}
\bibfield{author}{\bibinfo{person}{Haochen Chen}, \bibinfo{person}{Bryan
  Perozzi}, \bibinfo{person}{Yifan Hu}, {and} \bibinfo{person}{Steven Skiena}.}
  \bibinfo{year}{2018}\natexlab{}.
\newblock \showarticletitle{Harp: Hierarchical representation learning for
  networks}. In \bibinfo{booktitle}{\emph{AAAI}}.
\newblock


\bibitem[\protect\citeauthoryear{Chung}{Chung}{1997}]%
        {chung1997spectral}
\bibfield{author}{\bibinfo{person}{Fan~RK Chung}.}
  \bibinfo{year}{1997}\natexlab{}.
\newblock \bibinfo{booktitle}{\emph{Spectral graph theory}}.
\newblock Number~92. \bibinfo{publisher}{AMS}.
\newblock


\bibitem[\protect\citeauthoryear{Clauset, Moore, and Newman}{Clauset
  et~al\mbox{.}}{2008}]%
        {clauset2008hierarchical}
\bibfield{author}{\bibinfo{person}{Aaron Clauset}, \bibinfo{person}{Cristopher
  Moore}, {and} \bibinfo{person}{Mark~EJ Newman}.}
  \bibinfo{year}{2008}\natexlab{}.
\newblock \showarticletitle{Hierarchical structure and the prediction of
  missing links in networks}.
\newblock \bibinfo{journal}{\emph{Nature}} \bibinfo{volume}{453},
  \bibinfo{number}{7191} (\bibinfo{year}{2008}), \bibinfo{pages}{98--101}.
\newblock


\bibitem[\protect\citeauthoryear{Cui, Wang, Pei, and Zhu}{Cui
  et~al\mbox{.}}{2018}]%
        {cui2018survey}
\bibfield{author}{\bibinfo{person}{Peng Cui}, \bibinfo{person}{Xiao Wang},
  \bibinfo{person}{Jian Pei}, {and} \bibinfo{person}{Wenwu Zhu}.}
  \bibinfo{year}{2018}\natexlab{}.
\newblock \showarticletitle{A survey on network embedding}.
\newblock \bibinfo{journal}{\emph{IEEE TKDE}} \bibinfo{volume}{31},
  \bibinfo{number}{5} (\bibinfo{year}{2018}), \bibinfo{pages}{833--852}.
\newblock


\bibitem[\protect\citeauthoryear{Delvenne, Schaub, Yaliraki, and
  Barahona}{Delvenne et~al\mbox{.}}{2013}]%
        {delvenne2013stability}
\bibfield{author}{\bibinfo{person}{Jean-Charles Delvenne},
  \bibinfo{person}{Michael~T Schaub}, \bibinfo{person}{Sophia~N Yaliraki},
  {and} \bibinfo{person}{Mauricio Barahona}.} \bibinfo{year}{2013}\natexlab{}.
\newblock \showarticletitle{The stability of a graph partition: A
  dynamics-based framework for community detection}.
\newblock In \bibinfo{booktitle}{\emph{Dynamics On and Of Complex Networks}}.
  \bibinfo{publisher}{Springer}, \bibinfo{pages}{221--242}.
\newblock


\bibitem[\protect\citeauthoryear{Delvenne, Yaliraki, and Barahona}{Delvenne
  et~al\mbox{.}}{2010}]%
        {delvenne2010stability}
\bibfield{author}{\bibinfo{person}{J-C Delvenne}, \bibinfo{person}{Sophia~N
  Yaliraki}, {and} \bibinfo{person}{Mauricio Barahona}.}
  \bibinfo{year}{2010}\natexlab{}.
\newblock \showarticletitle{Stability of graph communities across time scales}.
\newblock \bibinfo{journal}{\emph{PNAS}} \bibinfo{volume}{107},
  \bibinfo{number}{29} (\bibinfo{year}{2010}), \bibinfo{pages}{12755--12760}.
\newblock


\bibitem[\protect\citeauthoryear{D{\'\i}az, Petit, and Serna}{D{\'\i}az
  et~al\mbox{.}}{2002}]%
        {diaz2002survey}
\bibfield{author}{\bibinfo{person}{Josep D{\'\i}az}, \bibinfo{person}{Jordi
  Petit}, {and} \bibinfo{person}{Maria Serna}.}
  \bibinfo{year}{2002}\natexlab{}.
\newblock \showarticletitle{A survey of graph layout problems}.
\newblock \bibinfo{journal}{\emph{ACM Computing Surveys (CSUR)}}
  \bibinfo{volume}{34}, \bibinfo{number}{3} (\bibinfo{year}{2002}),
  \bibinfo{pages}{313--356}.
\newblock


\bibitem[\protect\citeauthoryear{Dong, Chawla, and Swami}{Dong
  et~al\mbox{.}}{2017}]%
        {dong2017metapath2vec}
\bibfield{author}{\bibinfo{person}{Yuxiao Dong}, \bibinfo{person}{Nitesh~V
  Chawla}, {and} \bibinfo{person}{Ananthram Swami}.}
  \bibinfo{year}{2017}\natexlab{}.
\newblock \showarticletitle{metapath2vec: Scalable representation learning for
  heterogeneous networks}. In \bibinfo{booktitle}{\emph{SIGKDD}}.
\newblock


\bibitem[\protect\citeauthoryear{Donnat, Zitnik, Hallac, and Leskovec}{Donnat
  et~al\mbox{.}}{2018}]%
        {donnat2018learning}
\bibfield{author}{\bibinfo{person}{Claire Donnat}, \bibinfo{person}{Marinka
  Zitnik}, \bibinfo{person}{David Hallac}, {and} \bibinfo{person}{Jure
  Leskovec}.} \bibinfo{year}{2018}\natexlab{}.
\newblock \showarticletitle{Learning structural node embeddings via diffusion
  wavelets}. In \bibinfo{booktitle}{\emph{SIGKDD}}.
\newblock


\bibitem[\protect\citeauthoryear{Du, Wang, Song, Lu, and Wang}{Du
  et~al\mbox{.}}{2018}]%
        {du2018dynamic}
\bibfield{author}{\bibinfo{person}{Lun Du}, \bibinfo{person}{Yun Wang},
  \bibinfo{person}{Guojie Song}, \bibinfo{person}{Zhicong Lu}, {and}
  \bibinfo{person}{Junshan Wang}.} \bibinfo{year}{2018}\natexlab{}.
\newblock \showarticletitle{Dynamic Network Embedding: An Extended Approach for
  Skip-gram based Network Embedding}. In \bibinfo{booktitle}{\emph{IJCAI}}.
\newblock


\bibitem[\protect\citeauthoryear{Eckart and Young}{Eckart and Young}{1936}]%
        {eckart1936approximation}
\bibfield{author}{\bibinfo{person}{Carl Eckart} {and} \bibinfo{person}{Gale
  Young}.} \bibinfo{year}{1936}\natexlab{}.
\newblock \showarticletitle{The approximation of one matrix by another of lower
  rank}.
\newblock \bibinfo{journal}{\emph{Psychometrika}} \bibinfo{volume}{1},
  \bibinfo{number}{3} (\bibinfo{year}{1936}), \bibinfo{pages}{211--218}.
\newblock


\bibitem[\protect\citeauthoryear{Fan, Chang, Hsieh, Wang, and Lin}{Fan
  et~al\mbox{.}}{2008}]%
        {fan2008liblinear}
\bibfield{author}{\bibinfo{person}{Rong-En Fan}, \bibinfo{person}{Kai-Wei
  Chang}, \bibinfo{person}{Cho-Jui Hsieh}, \bibinfo{person}{Xiang-Rui Wang},
  {and} \bibinfo{person}{Chih-Jen Lin}.} \bibinfo{year}{2008}\natexlab{}.
\newblock \showarticletitle{LIBLINEAR: A library for large linear
  classification}.
\newblock \bibinfo{journal}{\emph{JMLR}} \bibinfo{volume}{9},
  \bibinfo{number}{Aug} (\bibinfo{year}{2008}), \bibinfo{pages}{1871--1874}.
\newblock


\bibitem[\protect\citeauthoryear{Gao, Huang, Cheng, Sun, Wang, and Li}{Gao
  et~al\mbox{.}}{2019}]%
        {gao2019link}
\bibfield{author}{\bibinfo{person}{Hua Gao}, \bibinfo{person}{Jianbin Huang},
  \bibinfo{person}{Qiang Cheng}, \bibinfo{person}{Heli Sun},
  \bibinfo{person}{Baoli Wang}, {and} \bibinfo{person}{He Li}.}
  \bibinfo{year}{2019}\natexlab{}.
\newblock \showarticletitle{Link prediction based on linear dynamical
  response}.
\newblock \bibinfo{journal}{\emph{Physica A: Statistical Mechanics and its
  Applications}}  \bibinfo{volume}{527} (\bibinfo{year}{2019}),
  \bibinfo{pages}{121397}.
\newblock


\bibitem[\protect\citeauthoryear{Girvan and Newman}{Girvan and Newman}{2002}]%
        {girvan2002community}
\bibfield{author}{\bibinfo{person}{Michelle Girvan} {and}
  \bibinfo{person}{Mark~EJ Newman}.} \bibinfo{year}{2002}\natexlab{}.
\newblock \showarticletitle{Community structure in social and biological
  networks}.
\newblock \bibinfo{journal}{\emph{PNAS}} \bibinfo{volume}{99},
  \bibinfo{number}{12} (\bibinfo{year}{2002}), \bibinfo{pages}{7821--7826}.
\newblock


\bibitem[\protect\citeauthoryear{Goyal and Ferrara}{Goyal and Ferrara}{2018}]%
        {goyal2018graph}
\bibfield{author}{\bibinfo{person}{Palash Goyal} {and} \bibinfo{person}{Emilio
  Ferrara}.} \bibinfo{year}{2018}\natexlab{}.
\newblock \showarticletitle{Graph embedding techniques, applications, and
  performance: A survey}.
\newblock \bibinfo{journal}{\emph{Knowledge-Based Systems}}
  \bibinfo{volume}{151} (\bibinfo{year}{2018}), \bibinfo{pages}{78--94}.
\newblock


\bibitem[\protect\citeauthoryear{Grover and Leskovec}{Grover and
  Leskovec}{2016}]%
        {grover2016node2vec}
\bibfield{author}{\bibinfo{person}{Aditya Grover} {and} \bibinfo{person}{Jure
  Leskovec}.} \bibinfo{year}{2016}\natexlab{}.
\newblock \showarticletitle{node2vec: Scalable feature learning for networks}.
  In \bibinfo{booktitle}{\emph{SIGKDD}}.
\newblock


\bibitem[\protect\citeauthoryear{Halko, Martinsson, and Tropp}{Halko
  et~al\mbox{.}}{2011}]%
        {halko2011finding}
\bibfield{author}{\bibinfo{person}{Nathan Halko}, \bibinfo{person}{Per-Gunnar
  Martinsson}, {and} \bibinfo{person}{Joel~A Tropp}.}
  \bibinfo{year}{2011}\natexlab{}.
\newblock \showarticletitle{Finding structure with randomness: Probabilistic
  algorithms for constructing approximate matrix decompositions}.
\newblock \bibinfo{journal}{\emph{SIAM review}} \bibinfo{volume}{53},
  \bibinfo{number}{2} (\bibinfo{year}{2011}), \bibinfo{pages}{217--288}.
\newblock


\bibitem[\protect\citeauthoryear{Hamilton, Ying, and Leskovec}{Hamilton
  et~al\mbox{.}}{2017}]%
        {hamilton2017representation}
\bibfield{author}{\bibinfo{person}{William~L Hamilton}, \bibinfo{person}{Rex
  Ying}, {and} \bibinfo{person}{Jure Leskovec}.}
  \bibinfo{year}{2017}\natexlab{}.
\newblock \showarticletitle{Representation learning on graphs: Methods and
  applications}.
\newblock \bibinfo{journal}{\emph{arXiv:1709.05584}} (\bibinfo{year}{2017}).
\newblock


\bibitem[\protect\citeauthoryear{Hinton and Roweis}{Hinton and Roweis}{2003}]%
        {hinton2003stochastic}
\bibfield{author}{\bibinfo{person}{Geoffrey~E Hinton} {and}
  \bibinfo{person}{Sam~T Roweis}.} \bibinfo{year}{2003}\natexlab{}.
\newblock \showarticletitle{Stochastic neighbor embedding}. In
  \bibinfo{booktitle}{\emph{NeurIPS}}.
\newblock


\bibitem[\protect\citeauthoryear{Javari, Derr, Esmailian, Tang, and
  Chang}{Javari et~al\mbox{.}}{2020}]%
        {javari2020rose}
\bibfield{author}{\bibinfo{person}{Amin Javari}, \bibinfo{person}{Tyler Derr},
  \bibinfo{person}{Pouya Esmailian}, \bibinfo{person}{Jiliang Tang}, {and}
  \bibinfo{person}{Kevin Chen-Chuan Chang}.} \bibinfo{year}{2020}\natexlab{}.
\newblock \showarticletitle{Rose: Role-based signed network embedding}. In
  \bibinfo{booktitle}{\emph{WebConf}}.
\newblock


\bibitem[\protect\citeauthoryear{Karrer and Newman}{Karrer and Newman}{2011}]%
        {karrer2011stochastic}
\bibfield{author}{\bibinfo{person}{Brian Karrer} {and} \bibinfo{person}{Mark~EJ
  Newman}.} \bibinfo{year}{2011}\natexlab{}.
\newblock \showarticletitle{Stochastic blockmodels and community structure in
  networks}.
\newblock \bibinfo{journal}{\emph{PRE}} \bibinfo{volume}{83},
  \bibinfo{number}{1} (\bibinfo{year}{2011}), \bibinfo{pages}{016107}.
\newblock


\bibitem[\protect\citeauthoryear{Khosla, Leonhardt, Nejdl, and Anand}{Khosla
  et~al\mbox{.}}{2019}]%
        {khosla2019node}
\bibfield{author}{\bibinfo{person}{Megha Khosla}, \bibinfo{person}{Jurek
  Leonhardt}, \bibinfo{person}{Wolfgang Nejdl}, {and} \bibinfo{person}{Avishek
  Anand}.} \bibinfo{year}{2019}\natexlab{}.
\newblock \showarticletitle{Node representation learning for directed graphs}.
  In \bibinfo{booktitle}{\emph{ECML-PKDD}}. Springer.
\newblock


\bibitem[\protect\citeauthoryear{Kingma and Ba}{Kingma and Ba}{2014}]%
        {kingma2014adam}
\bibfield{author}{\bibinfo{person}{Diederik~P Kingma} {and}
  \bibinfo{person}{Jimmy Ba}.} \bibinfo{year}{2014}\natexlab{}.
\newblock \showarticletitle{Adam: A method for stochastic optimization}.
\newblock \bibinfo{journal}{\emph{arXiv:1412.6980}} (\bibinfo{year}{2014}).
\newblock


\bibitem[\protect\citeauthoryear{Lambiotte, Delvenne, and Barahona}{Lambiotte
  et~al\mbox{.}}{2014}]%
        {lambiotte2014random}
\bibfield{author}{\bibinfo{person}{Renaud Lambiotte},
  \bibinfo{person}{Jean-Charles Delvenne}, {and} \bibinfo{person}{Mauricio
  Barahona}.} \bibinfo{year}{2014}\natexlab{}.
\newblock \showarticletitle{Random walks, Markov processes and the multiscale
  modular organization of complex networks}.
\newblock \bibinfo{journal}{\emph{IEEE Transactions on Network Science and
  Engineering}} \bibinfo{volume}{1}, \bibinfo{number}{2}
  (\bibinfo{year}{2014}), \bibinfo{pages}{76--90}.
\newblock


\bibitem[\protect\citeauthoryear{Lehoucq, Sorensen, and Yang}{Lehoucq
  et~al\mbox{.}}{1998}]%
        {lehoucq1998arpack}
\bibfield{author}{\bibinfo{person}{Richard~B Lehoucq}, \bibinfo{person}{Danny~C
  Sorensen}, {and} \bibinfo{person}{Chao Yang}.}
  \bibinfo{year}{1998}\natexlab{}.
\newblock \bibinfo{booktitle}{\emph{ARPACK users' guide: solution of
  large-scale eigenvalue problems with implicitly restarted Arnoldi methods}}.
\newblock \bibinfo{publisher}{SIAM}.
\newblock


\bibitem[\protect\citeauthoryear{Levy and Goldberg}{Levy and Goldberg}{2014}]%
        {levy2014neural}
\bibfield{author}{\bibinfo{person}{Omer Levy} {and} \bibinfo{person}{Yoav
  Goldberg}.} \bibinfo{year}{2014}\natexlab{}.
\newblock \showarticletitle{Neural word embedding as implicit matrix
  factorization}. In \bibinfo{booktitle}{\emph{NeurIPS}}.
\newblock


\bibitem[\protect\citeauthoryear{L{\"u} and Zhou}{L{\"u} and Zhou}{2011}]%
        {lu2011link}
\bibfield{author}{\bibinfo{person}{Linyuan L{\"u}} {and} \bibinfo{person}{Tao
  Zhou}.} \bibinfo{year}{2011}\natexlab{}.
\newblock \showarticletitle{Link prediction in complex networks: A survey}.
\newblock \bibinfo{journal}{\emph{Physica A: statistical mechanics and its
  applications}} \bibinfo{volume}{390}, \bibinfo{number}{6}
  (\bibinfo{year}{2011}), \bibinfo{pages}{1150--1170}.
\newblock


\bibitem[\protect\citeauthoryear{Mahoney}{Mahoney}{2011}]%
        {mahoney2011large}
\bibfield{author}{\bibinfo{person}{Matt Mahoney}.}
  \bibinfo{year}{2011}\natexlab{}.
\newblock \bibinfo{title}{Large text compression benchmark}.
\newblock
  \bibinfo{howpublished}{\url{https://www.mattmahoney.net/dc/textdata}}.
\newblock


\bibitem[\protect\citeauthoryear{Mikolov, Sutskever, Chen, Corrado, and
  Dean}{Mikolov et~al\mbox{.}}{2013}]%
        {mikolov2013distributed}
\bibfield{author}{\bibinfo{person}{Tomas Mikolov}, \bibinfo{person}{Ilya
  Sutskever}, \bibinfo{person}{Kai Chen}, \bibinfo{person}{Greg~S Corrado},
  {and} \bibinfo{person}{Jeff Dean}.} \bibinfo{year}{2013}\natexlab{}.
\newblock \showarticletitle{Distributed representations of words and phrases
  and their compositionality}. In \bibinfo{booktitle}{\emph{NeurIPS}}.
\newblock


\bibitem[\protect\citeauthoryear{Nguyen, Lee, Rossi, Ahmed, Koh, and
  Kim}{Nguyen et~al\mbox{.}}{2018}]%
        {nguyen2018continuous}
\bibfield{author}{\bibinfo{person}{Giang~Hoang Nguyen},
  \bibinfo{person}{John~Boaz Lee}, \bibinfo{person}{Ryan~A Rossi},
  \bibinfo{person}{Nesreen~K Ahmed}, \bibinfo{person}{Eunyee Koh}, {and}
  \bibinfo{person}{Sungchul Kim}.} \bibinfo{year}{2018}\natexlab{}.
\newblock \showarticletitle{Continuous-time dynamic network embeddings}. In
  \bibinfo{booktitle}{\emph{WebConf}}.
\newblock


\bibitem[\protect\citeauthoryear{Nickel and Kiela}{Nickel and Kiela}{2017}]%
        {nickel2017poincare}
\bibfield{author}{\bibinfo{person}{Maximillian Nickel} {and}
  \bibinfo{person}{Douwe Kiela}.} \bibinfo{year}{2017}\natexlab{}.
\newblock \showarticletitle{Poincar{\'e} embeddings for learning hierarchical
  representations}. In \bibinfo{booktitle}{\emph{NeurIPS}}.
\newblock


\bibitem[\protect\citeauthoryear{Ou, Cui, Pei, Zhang, and Zhu}{Ou
  et~al\mbox{.}}{2016}]%
        {ou2016asymmetric}
\bibfield{author}{\bibinfo{person}{Mingdong Ou}, \bibinfo{person}{Peng Cui},
  \bibinfo{person}{Jian Pei}, \bibinfo{person}{Ziwei Zhang}, {and}
  \bibinfo{person}{Wenwu Zhu}.} \bibinfo{year}{2016}\natexlab{}.
\newblock \showarticletitle{Asymmetric transitivity preserving graph
  embedding}. In \bibinfo{booktitle}{\emph{SIGKDD}}.
\newblock


\bibitem[\protect\citeauthoryear{Page, Brin, Motwani, and Winograd}{Page
  et~al\mbox{.}}{1999}]%
        {page1999pagerank}
\bibfield{author}{\bibinfo{person}{Lawrence Page}, \bibinfo{person}{Sergey
  Brin}, \bibinfo{person}{Rajeev Motwani}, {and} \bibinfo{person}{Terry
  Winograd}.} \bibinfo{year}{1999}\natexlab{}.
\newblock \bibinfo{booktitle}{\emph{The PageRank citation ranking: Bringing
  order to the web.}}
\newblock \bibinfo{type}{{T}echnical {R}eport}. \bibinfo{institution}{Stanford
  InfoLab}.
\newblock


\bibitem[\protect\citeauthoryear{Patokallio}{Patokallio}{2020}]%
        {openflight}
\bibfield{author}{\bibinfo{person}{Jani Patokallio}.}
  \bibinfo{year}{2020}\natexlab{}.
\newblock \bibinfo{title}{OpenFlights.org: Flight logging, mapping, stats and
  sharing}.
\newblock \bibinfo{howpublished}{\url{https://openflights.org/data.html}}.
\newblock


\bibitem[\protect\citeauthoryear{Perozzi, Al-Rfou, and Skiena}{Perozzi
  et~al\mbox{.}}{2014}]%
        {perozzi2014deepwalk}
\bibfield{author}{\bibinfo{person}{Bryan Perozzi}, \bibinfo{person}{Rami
  Al-Rfou}, {and} \bibinfo{person}{Steven Skiena}.}
  \bibinfo{year}{2014}\natexlab{}.
\newblock \showarticletitle{Deepwalk: Online learning of social
  representations}. In \bibinfo{booktitle}{\emph{SIGKDD}}.
\newblock


\bibitem[\protect\citeauthoryear{Perozzi, Kulkarni, Chen, and Skiena}{Perozzi
  et~al\mbox{.}}{2017}]%
        {perozzi2017don}
\bibfield{author}{\bibinfo{person}{Bryan Perozzi}, \bibinfo{person}{Vivek
  Kulkarni}, \bibinfo{person}{Haochen Chen}, {and} \bibinfo{person}{Steven
  Skiena}.} \bibinfo{year}{2017}\natexlab{}.
\newblock \showarticletitle{Don't Walk, Skip! Online learning of multi-scale
  network embeddings}. In \bibinfo{booktitle}{\emph{ASONAM}}.
\newblock


\bibitem[\protect\citeauthoryear{Qiu, Dong, Ma, Li, Wang, Wang, and Tang}{Qiu
  et~al\mbox{.}}{2019}]%
        {qiu2019netsmf}
\bibfield{author}{\bibinfo{person}{Jiezhong Qiu}, \bibinfo{person}{Yuxiao
  Dong}, \bibinfo{person}{Hao Ma}, \bibinfo{person}{Jian Li},
  \bibinfo{person}{Chi Wang}, \bibinfo{person}{Kuansan Wang}, {and}
  \bibinfo{person}{Jie Tang}.} \bibinfo{year}{2019}\natexlab{}.
\newblock \showarticletitle{Netsmf: Large-scale network embedding as sparse
  matrix factorization}. In \bibinfo{booktitle}{\emph{WebConf}}.
\newblock


\bibitem[\protect\citeauthoryear{Qiu, Dong, Ma, Li, Wang, and Tang}{Qiu
  et~al\mbox{.}}{2018}]%
        {qiu2018network}
\bibfield{author}{\bibinfo{person}{Jiezhong Qiu}, \bibinfo{person}{Yuxiao
  Dong}, \bibinfo{person}{Hao Ma}, \bibinfo{person}{Jian Li},
  \bibinfo{person}{Kuansan Wang}, {and} \bibinfo{person}{Jie Tang}.}
  \bibinfo{year}{2018}\natexlab{}.
\newblock \showarticletitle{Network embedding as matrix factorization: Unifying
  deepwalk, line, pte, and node2vec}. In \bibinfo{booktitle}{\emph{WSDM}}.
\newblock


\bibitem[\protect\citeauthoryear{Ravasz and Barab{\'a}si}{Ravasz and
  Barab{\'a}si}{2003}]%
        {ravasz2003hierarchical}
\bibfield{author}{\bibinfo{person}{Erzs{\'e}bet Ravasz} {and}
  \bibinfo{person}{Albert-L{\'a}szl{\'o} Barab{\'a}si}.}
  \bibinfo{year}{2003}\natexlab{}.
\newblock \showarticletitle{Hierarchical organization in complex networks}.
\newblock \bibinfo{journal}{\emph{PRE}} \bibinfo{volume}{67},
  \bibinfo{number}{2} (\bibinfo{year}{2003}), \bibinfo{pages}{026112}.
\newblock


\bibitem[\protect\citeauthoryear{Ruelle}{Ruelle}{2004}]%
        {ruelle2004thermodynamic}
\bibfield{author}{\bibinfo{person}{David Ruelle}.}
  \bibinfo{year}{2004}\natexlab{}.
\newblock \bibinfo{booktitle}{\emph{Thermodynamic formalism: the mathematical
  structure of equilibrium statistical mechanics}}.
\newblock \bibinfo{publisher}{Cambridge University Press}.
\newblock


\bibitem[\protect\citeauthoryear{Schaub, Delvenne, Lambiotte, and
  Barahona}{Schaub et~al\mbox{.}}{2019}]%
        {schaub2019multiscale}
\bibfield{author}{\bibinfo{person}{Michael~T Schaub},
  \bibinfo{person}{Jean-Charles Delvenne}, \bibinfo{person}{Renaud Lambiotte},
  {and} \bibinfo{person}{Mauricio Barahona}.} \bibinfo{year}{2019}\natexlab{}.
\newblock \showarticletitle{Multiscale dynamical embeddings of complex
  networks}.
\newblock \bibinfo{journal}{\emph{PRE}} \bibinfo{volume}{99},
  \bibinfo{number}{6} (\bibinfo{year}{2019}), \bibinfo{pages}{062308}.
\newblock


\bibitem[\protect\citeauthoryear{Seshadhri, Sharma, Stolman, and
  Goel}{Seshadhri et~al\mbox{.}}{2020}]%
        {seshadhri2020impossibility}
\bibfield{author}{\bibinfo{person}{C Seshadhri}, \bibinfo{person}{Aneesh
  Sharma}, \bibinfo{person}{Andrew Stolman}, {and} \bibinfo{person}{Ashish
  Goel}.} \bibinfo{year}{2020}\natexlab{}.
\newblock \showarticletitle{The impossibility of low-rank representations for
  triangle-rich complex networks}.
\newblock \bibinfo{journal}{\emph{PNAS}} \bibinfo{volume}{117},
  \bibinfo{number}{11} (\bibinfo{year}{2020}), \bibinfo{pages}{5631--5637}.
\newblock


\bibitem[\protect\citeauthoryear{Srinivasan and Ribeiro}{Srinivasan and
  Ribeiro}{2019}]%
        {srinivasan2019equivalence}
\bibfield{author}{\bibinfo{person}{Balasubramaniam Srinivasan} {and}
  \bibinfo{person}{Bruno Ribeiro}.} \bibinfo{year}{2019}\natexlab{}.
\newblock \showarticletitle{On the equivalence between positional node
  embeddings and structural graph representations}. In
  \bibinfo{booktitle}{\emph{ICLR}}.
\newblock


\bibitem[\protect\citeauthoryear{Strehl and Ghosh}{Strehl and Ghosh}{2002}]%
        {strehl2002cluster}
\bibfield{author}{\bibinfo{person}{Alexander Strehl} {and}
  \bibinfo{person}{Joydeep Ghosh}.} \bibinfo{year}{2002}\natexlab{}.
\newblock \showarticletitle{Cluster ensembles---a knowledge reuse framework for
  combining multiple partitions}.
\newblock \bibinfo{journal}{\emph{JMLR}} \bibinfo{volume}{3},
  \bibinfo{number}{Dec} (\bibinfo{year}{2002}), \bibinfo{pages}{583--617}.
\newblock


\bibitem[\protect\citeauthoryear{{\v{S}}ubelj and Bajec}{{\v{S}}ubelj and
  Bajec}{2013}]%
        {vsubelj2013model}
\bibfield{author}{\bibinfo{person}{Lovro {\v{S}}ubelj} {and}
  \bibinfo{person}{Marko Bajec}.} \bibinfo{year}{2013}\natexlab{}.
\newblock \showarticletitle{Model of complex networks based on citation
  dynamics}. In \bibinfo{booktitle}{\emph{WebConf}}.
\newblock


\bibitem[\protect\citeauthoryear{Tang, Qu, Wang, Zhang, Yan, and Mei}{Tang
  et~al\mbox{.}}{2015}]%
        {tang2015line}
\bibfield{author}{\bibinfo{person}{Jian Tang}, \bibinfo{person}{Meng Qu},
  \bibinfo{person}{Mingzhe Wang}, \bibinfo{person}{Ming Zhang},
  \bibinfo{person}{Jun Yan}, {and} \bibinfo{person}{Qiaozhu Mei}.}
  \bibinfo{year}{2015}\natexlab{}.
\newblock \showarticletitle{Line: Large-scale information network embedding}.
  In \bibinfo{booktitle}{\emph{WebConf}}.
\newblock


\bibitem[\protect\citeauthoryear{Tang and Liu}{Tang and Liu}{2009}]%
        {tang2009relational}
\bibfield{author}{\bibinfo{person}{Lei Tang} {and} \bibinfo{person}{Huan Liu}.}
  \bibinfo{year}{2009}\natexlab{}.
\newblock \showarticletitle{Relational learning via latent social dimensions}.
  In \bibinfo{booktitle}{\emph{SIGKDD}}.
\newblock


\bibitem[\protect\citeauthoryear{Tang, Rajan, and Narayanan}{Tang
  et~al\mbox{.}}{2009}]%
        {tang2009large}
\bibfield{author}{\bibinfo{person}{Lei Tang}, \bibinfo{person}{Suju Rajan},
  {and} \bibinfo{person}{Vijay~K Narayanan}.} \bibinfo{year}{2009}\natexlab{}.
\newblock \showarticletitle{Large scale multi-label classification via
  metalabeler}. In \bibinfo{booktitle}{\emph{WebConf}}.
\newblock


\bibitem[\protect\citeauthoryear{Tsitsulin, Mottin, Karras, and
  M{\"u}ller}{Tsitsulin et~al\mbox{.}}{2018}]%
        {tsitsulin2018verse}
\bibfield{author}{\bibinfo{person}{Anton Tsitsulin}, \bibinfo{person}{Davide
  Mottin}, \bibinfo{person}{Panagiotis Karras}, {and} \bibinfo{person}{Emmanuel
  M{\"u}ller}.} \bibinfo{year}{2018}\natexlab{}.
\newblock \showarticletitle{Verse: Versatile graph embeddings from similarity
  measures}. In \bibinfo{booktitle}{\emph{WebConf}}.
\newblock


\bibitem[\protect\citeauthoryear{Tsoumakas, Katakis, and Vlahavas}{Tsoumakas
  et~al\mbox{.}}{2009}]%
        {tsoumakas2009mining}
\bibfield{author}{\bibinfo{person}{Grigorios Tsoumakas},
  \bibinfo{person}{Ioannis Katakis}, {and} \bibinfo{person}{Ioannis Vlahavas}.}
  \bibinfo{year}{2009}\natexlab{}.
\newblock \showarticletitle{Mining multi-label data}.
\newblock In \bibinfo{booktitle}{\emph{Data mining and knowledge discovery
  handbook}}. \bibinfo{publisher}{Springer}, \bibinfo{pages}{667--685}.
\newblock


\bibitem[\protect\citeauthoryear{Wang, Cui, and Zhu}{Wang
  et~al\mbox{.}}{2016}]%
        {wang2016structural}
\bibfield{author}{\bibinfo{person}{Daixin Wang}, \bibinfo{person}{Peng Cui},
  {and} \bibinfo{person}{Wenwu Zhu}.} \bibinfo{year}{2016}\natexlab{}.
\newblock \showarticletitle{Structural deep network embedding}. In
  \bibinfo{booktitle}{\emph{SIGKDD}}.
\newblock


\bibitem[\protect\citeauthoryear{Wang, Tang, Aggarwal, Chang, and Liu}{Wang
  et~al\mbox{.}}{2017}]%
        {wang2017signed}
\bibfield{author}{\bibinfo{person}{Suhang Wang}, \bibinfo{person}{Jiliang
  Tang}, \bibinfo{person}{Charu Aggarwal}, \bibinfo{person}{Yi Chang}, {and}
  \bibinfo{person}{Huan Liu}.} \bibinfo{year}{2017}\natexlab{}.
\newblock \showarticletitle{Signed network embedding in social media}. In
  \bibinfo{booktitle}{\emph{SDM}}.
\newblock


\bibitem[\protect\citeauthoryear{Xin, Chen, Chen, and Zhao}{Xin
  et~al\mbox{.}}{2019}]%
        {xin2019marc}
\bibfield{author}{\bibinfo{person}{Zhenghua Xin}, \bibinfo{person}{Jie Chen},
  \bibinfo{person}{Guolong Chen}, {and} \bibinfo{person}{Shu Zhao}.}
  \bibinfo{year}{2019}\natexlab{}.
\newblock \showarticletitle{Marc: Multi-Granular Representation Learning for
  Networks Based on the 3-Clique}.
\newblock \bibinfo{journal}{\emph{IEEE Access}}  \bibinfo{volume}{7}
  (\bibinfo{year}{2019}), \bibinfo{pages}{141715--141727}.
\newblock


\bibitem[\protect\citeauthoryear{Zhang, Qiu, Yi, and Song}{Zhang
  et~al\mbox{.}}{2018b}]%
        {zhang2018scalable}
\bibfield{author}{\bibinfo{person}{Hongming Zhang}, \bibinfo{person}{Liwei
  Qiu}, \bibinfo{person}{Lingling Yi}, {and} \bibinfo{person}{Yangqiu Song}.}
  \bibinfo{year}{2018}\natexlab{b}.
\newblock \showarticletitle{Scalable Multiplex Network Embedding}. In
  \bibinfo{booktitle}{\emph{IJCAI}}.
\newblock


\bibitem[\protect\citeauthoryear{Zhang, Cui, Wang, Pei, Yao, and Zhu}{Zhang
  et~al\mbox{.}}{2018a}]%
        {zhang2018arbitrary}
\bibfield{author}{\bibinfo{person}{Ziwei Zhang}, \bibinfo{person}{Peng Cui},
  \bibinfo{person}{Xiao Wang}, \bibinfo{person}{Jian Pei},
  \bibinfo{person}{Xuanrong Yao}, {and} \bibinfo{person}{Wenwu Zhu}.}
  \bibinfo{year}{2018}\natexlab{a}.
\newblock \showarticletitle{Arbitrary-order proximity preserved network
  embedding}. In \bibinfo{booktitle}{\emph{SIGKDD}}.
\newblock


\bibitem[\protect\citeauthoryear{Zhou, Liu, Liu, Liu, and Gao}{Zhou
  et~al\mbox{.}}{2017}]%
        {zhou2017scalable}
\bibfield{author}{\bibinfo{person}{Chang Zhou}, \bibinfo{person}{Yuqiong Liu},
  \bibinfo{person}{Xiaofei Liu}, \bibinfo{person}{Zhongyi Liu}, {and}
  \bibinfo{person}{Jun Gao}.} \bibinfo{year}{2017}\natexlab{}.
\newblock \showarticletitle{Scalable graph embedding for asymmetric proximity}.
  In \bibinfo{booktitle}{\emph{AAAI}}.
\newblock


\end{thebibliography}
\clearpage
\appendix

\section{Appendix}

\subsection{Proof of the Relationship between PMI Similarity and Skip-gram Based Models}\label{subsec::relationship}
Here, we show the close relationship between the PMI similarity and existing embedding models. As proved in \cite{qiu2018network}, LINE and DeepWalk implicitly factorize the following matrices: 
\begin{align}
    R_{LINE} = \log (\vol(\mathcal{G}) D^{-1}AD^{-1}) - \log b\\
    R_{DW} = \log (\vol(G) (\frac{1}{T}\sum_{r=1}^T (D^{-1}A)^r) D^{-1}) - \log b
\end{align}
where $b$ and $T$ are the number of negative samples and the context window size in skip-gram respectively, and $\vol(\mathcal{G}) = \sum_{u} \deg(u)$. Now, consider the standard random-walk where $M = D^{-1}A$ and $\Pi = D/\vol(\mathcal{G})$. We notice the following equivalence:
\begin{lemma}
The PMI similarity matrix $R(\tau)$ in \autoref{eqn::pmi_matrix} for the standard random-walk can be expressed as
\begin{equation}
    R(\tau) = \log(\vol(\mathcal{G})(D^{-1}A)^{\tau}D^{-1})
\end{equation}
\end{lemma}
\begin{proof}
We use $\circ$ and $\oslash$ to represent elementwise multiplication and division between matrices. We have:
\begin{equation}
    \begin{aligned}
    R(\tau) & = \log (\Pi M^{\tau}) - \log(\pi \pi^T) \\
    & = \log ( (\pi, \dotsc, \pi) \circ M^{\tau}) - \log ( (\pi, \dotsc, \pi) \circ (\pi, \dotsc, \pi)^T) \\
    & = \log (M^{\tau}) - \log ((\pi, \dotsc, \pi)^T) \\
    & = \log (M^{\tau} \oslash (\pi, \dotsc, \pi)^T )\\
    & = \log(M^{\tau} \Pi^{-1}) \\ 
    & = \log(\vol(\mathcal{G})(D^{-1}A)^{\tau}D^{-1})
    \end{aligned}
\end{equation}
\end{proof}
Thus, $R_{LINE}$ factorizes a shifted version of the similarity at $\tau\!=\!1$:
\begin{equation}
    R_{LINE} = R(1) - \log b
\end{equation}

And $R_{DW}$ factorizes a smooth approximation of the average (\emph{log-mean-exp}) for the shifted similarity with $t$ from $1$ to $T$:
\begin{equation}
    R_{DW} = \log \left(\frac{1}{T} \sum_{\tau=1}^T \exp(R(\tau))\right) - \log b
\end{equation}

\subsection{Proof of Theorem \ref{thm::sampling_stab}}
\begin{proof}
We first show that, for $b=1$, Equations \ref{eqn::density_autocov_one} and \ref{eqn::density_autocov_zero}
can be derived from the definition of autocovariance and the Bayes Theorem. The corpus $\mathcal{D}$ is composed of samples from the random-walk process, and thus, for autocovariance-based embeddings:
\begin{equation}
\begin{aligned}
    p(u,v|z=1) &= \pi_u p(x(t{+}\tau){=}v|x(t){=}u))\\
                    &= \mathbf{u}_u^T\mathbf{v}_v^{} + \pi_u\pi_v
\end{aligned}
\end{equation}
For pairs that are not in the corpus (negative samples):
\begin{equation}
    p(u,v|z=0) = \pi_u\pi_v
\end{equation}
Because $b\!=\!1$, $p(z\!=\!0)\!=\!p(z\!=\!1)$, and $p(u,v)\!=\!  \mathbf{u}_u^T\mathbf{u}_v^{}\!+\!2\pi_u\pi_v$. From the Bayes Theorem, we get \autoref{eqn::density_autocov_one} and \autoref{eqn::density_autocov_zero}.
We will assume that \autoref{eqn::auto_cov_pair_neg} can be computed  exactly from samples:
\begin{equation}
\mathbf{u}_u^T\mathbf{v}_v^{} = \frac{\#(u,v)}{b|\mathcal{D}|} - \frac{\#(u)\#(v)}{|\mathcal{D}|^2}
\label{eqn::to_prove_sampling}
\end{equation}
Moreover, similar to \cite{levy2014neural}, we will simplify \autoref{eqn::sampling_likelihood} as follows:
\begin{equation*}
\begin{aligned}
\ell\!=&\!\sum_{u,v}\!\#(u,v)\log(p(z\!=\!1|u,v))\!+\!b\sum_{u,v}\!\#(u,v)\mathbb{E}_{w\sim \pi}\log(p(z\!=\!0|u,w))\\
=&\sum_{u,v} \#(u,v)\log(p(z\!=\!1|u,v))+b\sum_{u} \#(u)\frac{\#(v)}{|\mathcal{D}|}\log(p(z\!=\!0|u,v))\\
&+b\sum_{u,w\neq v} \#(u)\frac{\#(w)}{|\mathcal{D}|}\log(p(z\!=\!0|u,w))
\end{aligned}
\end{equation*}
where $\#(v)=\sum_{w}\#(v,w)$.

For a large enough number of dimensions, we can minimize the above Equation for each pair $(u,v)$:
\begin{equation}
\ell_{u,v} = \#(u,v)\log(p(z\!=\!1|u,v))+b\#(u)\frac{\#(v)}{|\mathcal{D}|}\log(p(z\!=\!0|u,v))
\label{eqn::pairwise_loss}
\end{equation}
Now, we plug the conditional probabilities $p(z\!=\!1|u,v)$ and $p(z\!=\!0|u,v)$ from \autoref{eqn::density_autocov_one} and \autoref{eqn::density_autocov_zero} into $\ell_{u,v}$, and compute its derivative with respect to $x=\mathbf{u}_u^T\mathbf{v}_v^{}$, and we have
\begin{equation*}
\begin{aligned}
\frac{\partial \ell_{u,v}}{\partial x} &= \frac{-\#(u,v)\pi_u\pi_v}{(x+\pi_u\pi_v)(x+(b+1)\pi_u\pi_v)}+\frac{b\#(u)\#(v)}{|\mathcal{D}|(x+(b+1)\pi_u\pi_v)}
\end{aligned}
\end{equation*}
where we have conveniently dropped the function $\rho$.

Setting the derivative to zero, we get \autoref{eqn::to_prove_sampling}. We emphasize that similar to \cite{levy2014neural}, our theorem only holds when $d$ is large enough to allow embeddings to be optimal pairwise, which can only be guaranteed in the general case when $d=n$. 
\end{proof}
\subsection{Additional Figures}

\autoref{fig::multiscale_community_detection} and \autoref{fig::multiscale_link_prediction_bad} (both referred in Section \ref{subsubsec::multiscale}) show community detection and link prediction performance for varying Markov time. In \autoref{fig::directed_vs_undirected_link_prediction} and \autoref{fig::directed_vs_undirected_classification} (both referred in Section \ref{subsubsec::directed_vs_undirected}) we show link prediction and node classification results for directed and undirected embeddings with PageRank. And \autoref{fig::correlation} (referred in Section \ref{subsubsec::factorization_vs_sampling}) shows the correlation between entries of similarity matrices and the reconstruction from the embeddings.

\begin{figure}[H]
  \centering
  \includegraphics[width=\columnwidth]{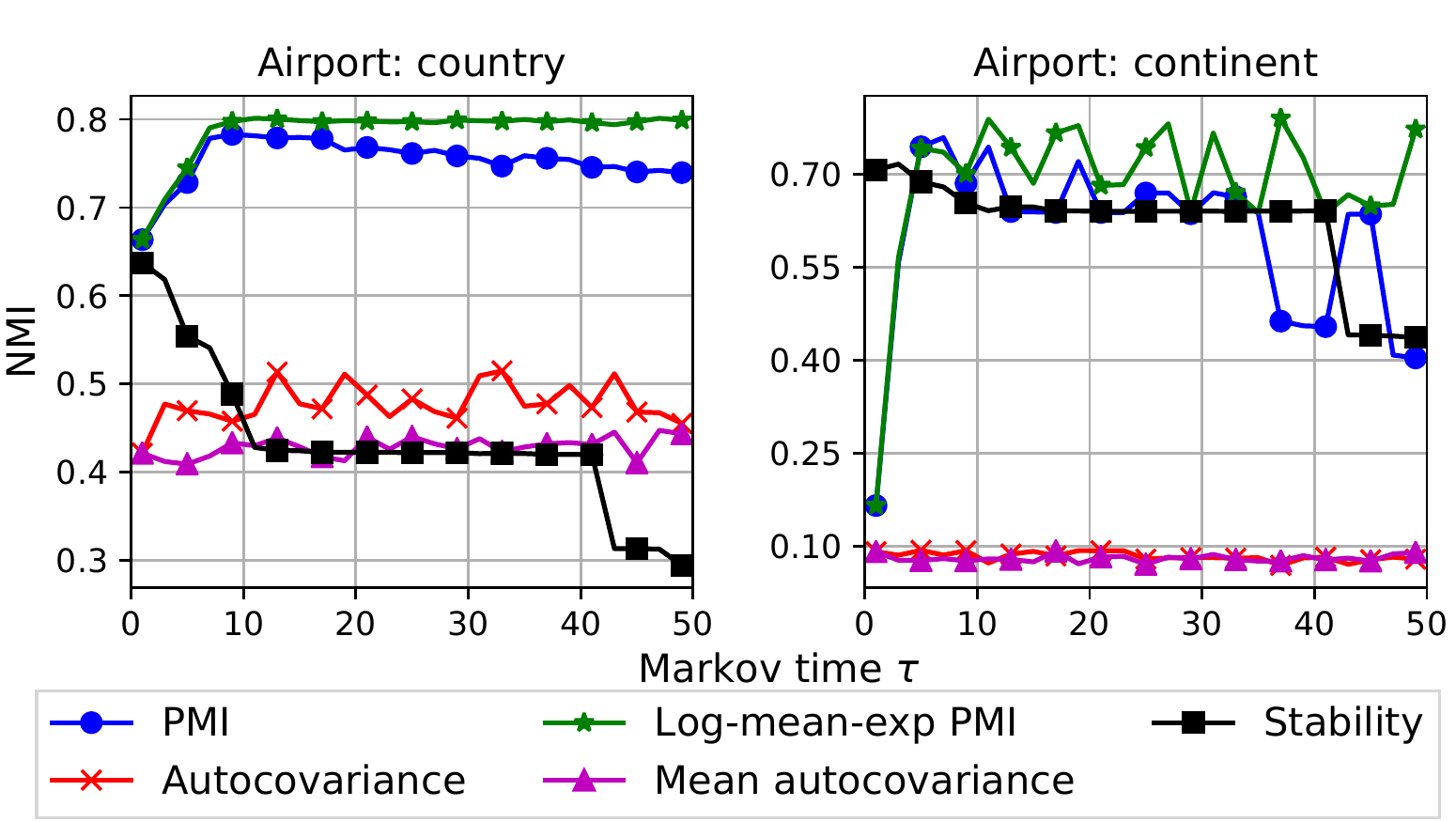}
  \caption{
  Community detection for PMI, autocovariance, their moving means, and Markov Stability \cite{delvenne2010stability} on varying Markov times. 
  \emph{Log-mean-exp} PMI outperforms autocovariance and Markov Stability for country and continent levels.
  }
  \label{fig::multiscale_community_detection}
\end{figure}
\clearpage

\begin{figure*}[!htb]
\begin{multicols}{2}
  \centering
  \includegraphics[width=\columnwidth]{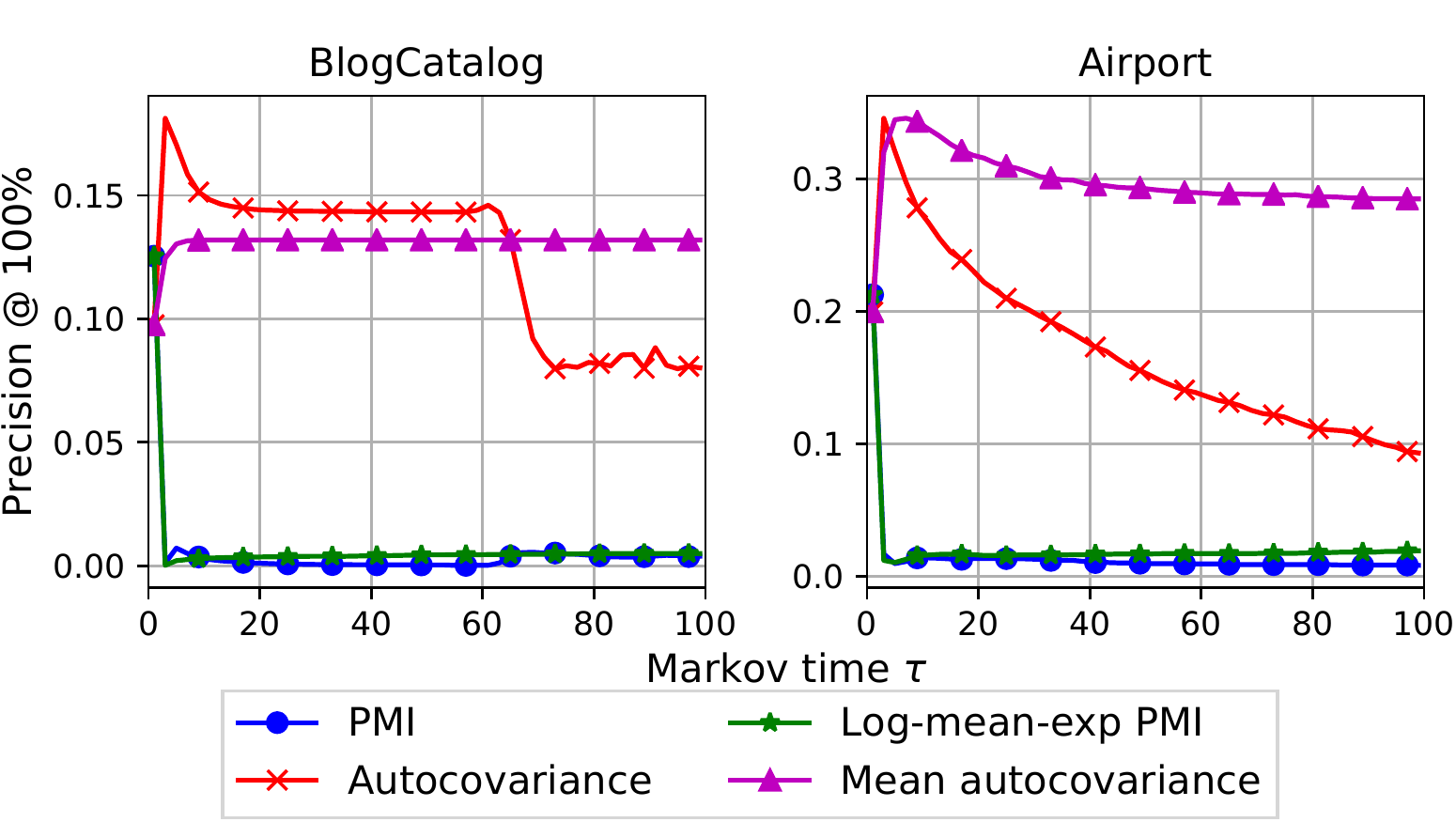}
  \caption{Link prediction for PMI, autocovariance, and their moving means on varying Markov times. 
  Neither \emph{log-mean-exp} PMI nor mean autocovariance increases the peak performance. 
  }
  \label{fig::multiscale_link_prediction_bad}
  
  \centering
  \includegraphics[width=\columnwidth]{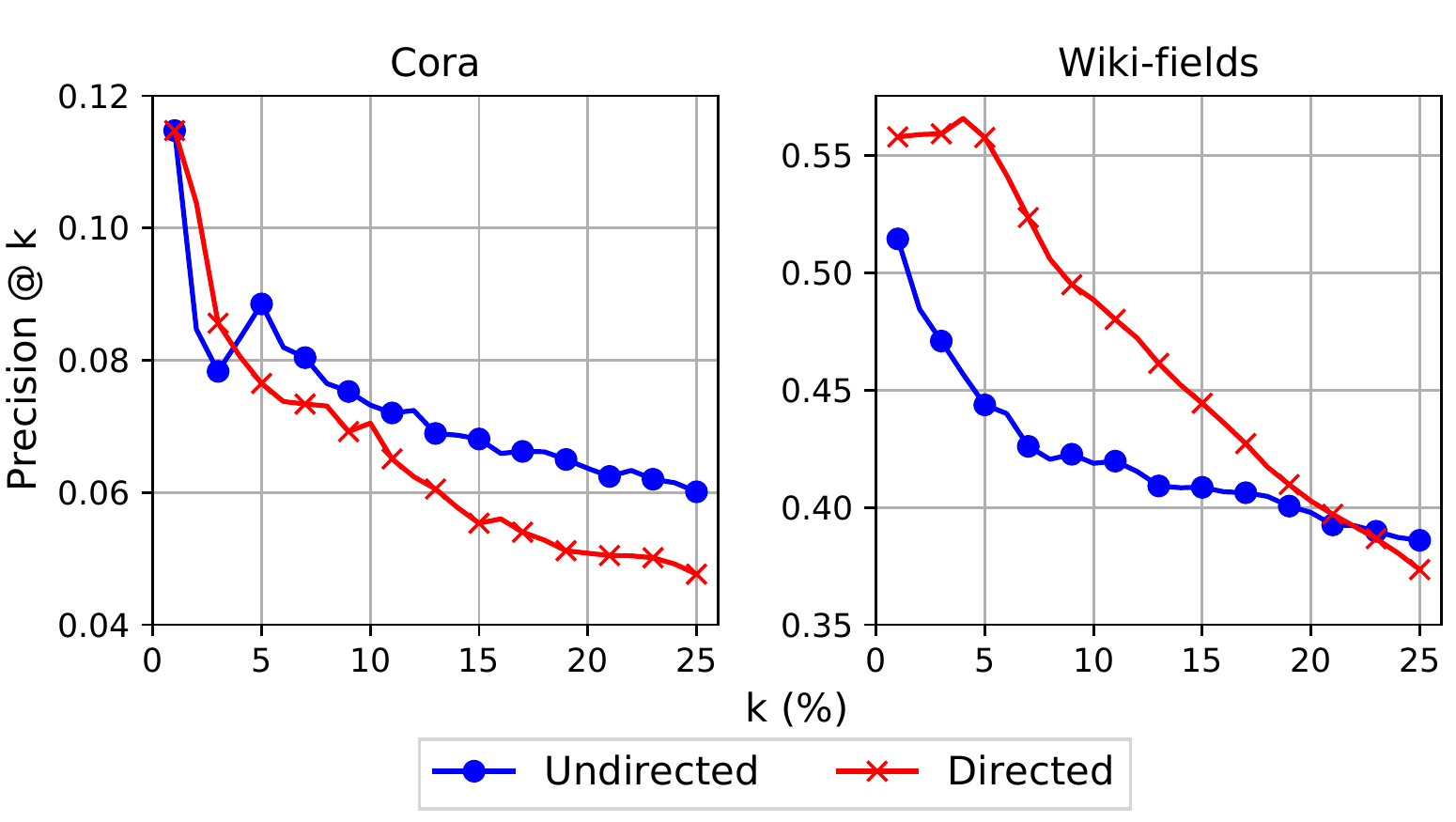}
  \caption{Link prediction for directed and undirected embeddings with PageRank on varying top pairs. 
  Directed embedding outperforms undirected embedding for the very top ranked edges. 
  }
  \label{fig::directed_vs_undirected_link_prediction}
\end{multicols}

  \centering
  \includegraphics[width=\textwidth]{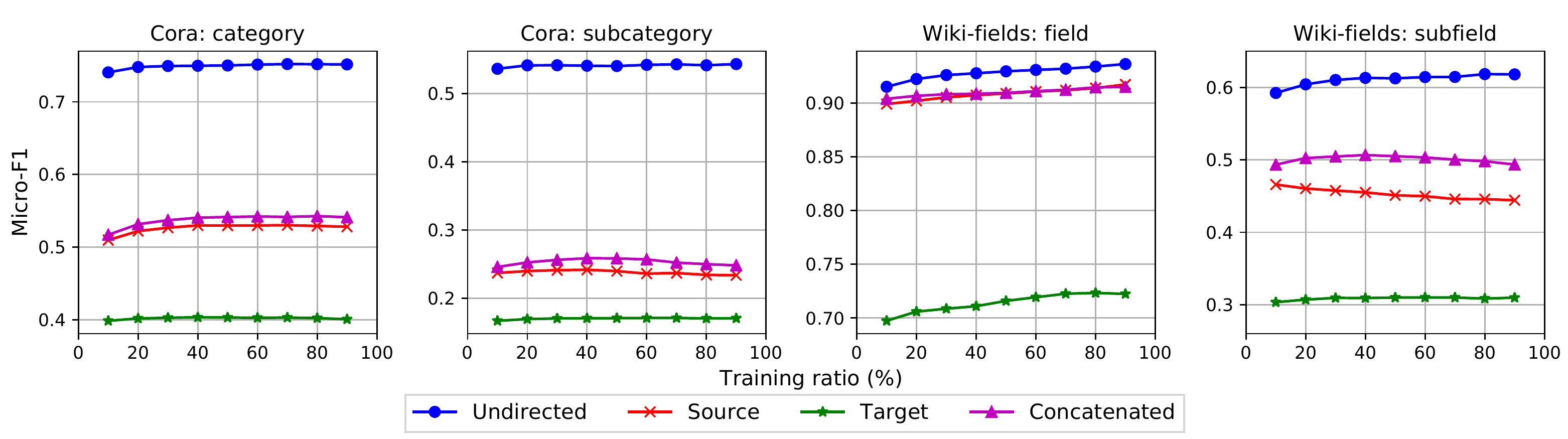}
  \caption{Node classification for directed and undirected embeddings with PageRank on varying training ratios. \emph{Macro-F1} scores are not shown here as they follow similar patterns as \emph{Micro-F1} scores.
  The undirected embedding consistently outperforms both source, target and concatenated directed embeddings in both datasets.
  }
  \label{fig::directed_vs_undirected_classification}
\centering
\includegraphics[width=\textwidth]{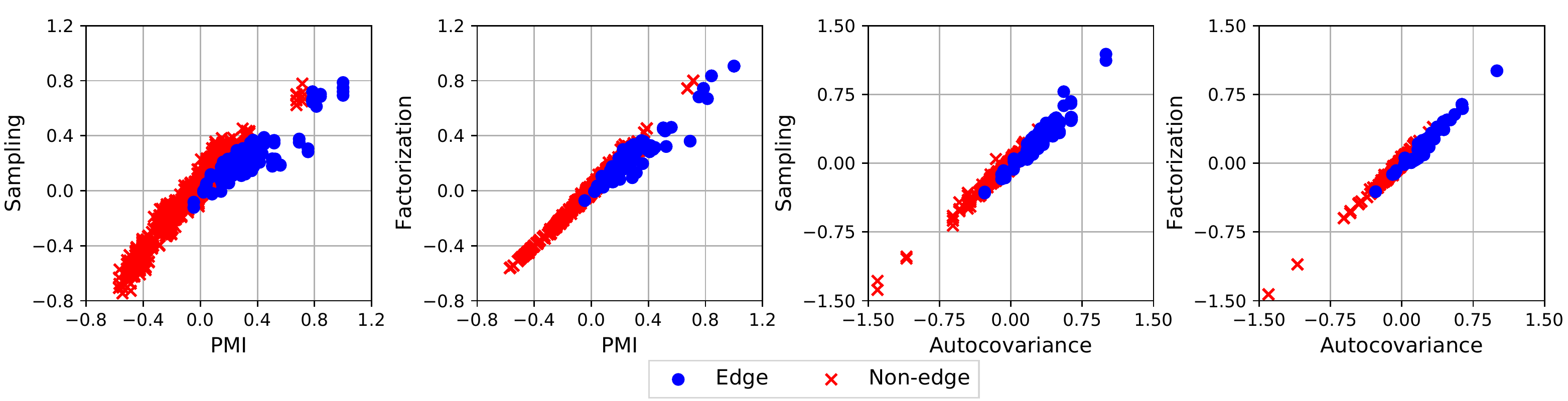}

\caption{Correlation between entries of the similarity matrices (PMI and autocovariance) and the corresponding dot product reconstruction from the embeddings generated via sampling and matrix factorization algorithms. The results are generated using Zachary's karate club network. Both edge and non-edge pairs of nodes are shown. While a clear correlation can be noticed in all cases, matrix factorization methods provide a better approximation of the similarity metrics. \label{fig::correlation}
}
\end{figure*}

\end{document}